\documentclass[11pt]{article}

\usepackage{microtype}
\usepackage{graphicx}
\usepackage{subcaption}
\usepackage{natbib}
\usepackage{xcolor}

\usepackage{algorithm}
\usepackage{algorithmic}

\usepackage{amsmath, amssymb, amsthm, bm}
\usepackage{multirow}
\usepackage{float}
\renewcommand{\vec}{\text{vec}}
\newtheorem{assump}{Assumption}
\newtheorem{thm}{Theorem}
\newtheorem{lem}{Lemma}

\newtheorem{defi}{Definition}

\DeclareMathOperator*{\argmin}{arg\,min}
\DeclareMathOperator*{\argmax}{arg\,max}

\usepackage{hyperref}

\title{\bf Matrix Completion using Kronecker Product Approximation}
\author{Chencheng Cai$^a$, Rong Chen$^b$ and Han Xiao$^b$\footnote{Chencheng Cai is a postdoctoral         fellow 
    at Department of Statistical Science, Fox School of Business, Temple University,  Philadelphia, PA 19122. E-mail:
    chencheng.cai@temple.edu.
Rong Chen is Professor, Department of
    Statistics, Rutgers University, Piscataway, NJ 08854. E-mail:
    rongchen@stat.rutgers.edu. Han Xiao is Associate Professor,
    Department of Statistics, Rutgers
    University, Piscataway, NJ 08854. E-mail:
    hxiao@stat.rutgers.edu.
Han Xiao is the
    corresponding author. Chen's research is supported
in part by National Science Foundation
grants DMS-1503409, DMS-1737857, IIS-1741390 and
CCF-1934924. Xiao's research is supported in part by National Science Foundation grants DMS-1454817, ATD-2027855, and a research grant from NEC Labs America.
}
\\ \vspace{0.2cm} {$^a$Temple University, $^b$Rutgers University}\\}
\date{}

\begin{document}
\maketitle

\begin{abstract}
A matrix completion problem is to recover the missing entries in a partially observed matrix. Most of the existing matrix completion methods assume a low rank structure of the underlying complete matrix. In this paper, we introduce an alternative and more general form of the underlying complete matrix, which assumes a low Kronecker rank instead of a low regular rank, but includes the latter as a special case. The extra flexibility allows for a much more parsimonious representation of the underlying matrix, but also raises the challenge of determining the proper Kronecker product configuration to be used. We find that the configuration can be identified using the mean squared error criterion as well as a modified cross-validation criterion. We establish the consistency of this procedure under suitable conditions on the signal-to-noise ratio. A aggregation procedure is also proposed to deal with special missing patterns and complex underlying structures. Both numerical and empirical studies are carried out to demonstrate the performance of the new method. 
\end{abstract}


\section{Introduction}
Many applications involve observations in a matrix form, and most likely of large dimensions. 
When the observed matrix is the sum of a signal matrix and a noise matrix, or is partially observed, 
a common approach in machine learning and statistics is to assume that the underlying signal matrix has a rank that is much smaller than its dimension. The low rank structure represents the interaction between matrix entries with a smaller number of parameters and reveals the core factors that drive and control the high dimensional observations, resulting in significant dimension reduction. Such a low rank assumption also makes it possible to recover the missing entries in a partially observed matrix, known as the {\it matrix completion problem}. Matrix completion has broad and important applications, 
including collaborative filtering \citep{goldberg1992using}, global positioning \citep{biswas2006semidefinite} and remote sensing \citep{schmidt1986multiple}, among many others. One of the most famous examples is the Netflix recommendation system contest \citep{bennett2007netflix}, in which the winning algorithm recovers the movie-rating matrix by a rank one matrix based on the observed ratings. 

Two different settings of matrix completion problems have been considered in the literature. One is the exact matrix completion problem whose goal is to recover the original matrix \emph{exactly} when a portion of the matrix entries is missing. When the original matrix rank is known, it can be recovered through the alternating minimization algorithm proposed by \citet{jain2013low} under certain conditions. When the matrix rank is unknown, it is still possible to exactly recover the matrix through nuclear norm optimization \citep{candes2009exact, candes2010power}. The nuclear norm optimization approach can also be applied to tensor completion problems whose goal is to recover a tensor structure \citep{yuan2016tensor}. The second setting considers the circumstances when the observed entries are corrupted by noises while at the same time a portion of
the entries is missing. It is known as the stable matrix completion problem.
\citet{candes2010matrix} extends the nuclear norm optimization approach to the stable matrix completion problem by relaxing the constraint.
Assuming the matrix rank is known, \citet{keshavan2010matrix} approaches the problem using a combination of spectral techniques and manifold optimization. Specifically for stable rank one matrix completion problem, \citet{cosse2017stable} proposes to solve it using two rounds of semi-definite programming relaxation. Note that the alternating minimization algorithm in \citet{jain2013low} is applicable for the stable matrix completion problem as well.

It is observed that in many applications of image analysis, signal processing and quantum computing, the high dimensional data in matrix form often has a low-rank structure in terms of Kronecker product decomposition instead of singular value decomposition \citep{Werner2008On, Duarte2012Kronecker, Kamm1998Kronecker}.
Approximating a matrix with a sum of a small number of matrices in Kronecker
product form is an extension of the low rank approximation with a sum of rank one
matrices. The flexibility provides an alternative approach for matrix
completion. The key challenging factor of the approach is to
determine 
the Kronecker product's configuration, which is the dimensions of the two matrices of the product.

In this article, we consider the matrix completion problem under the setting that the signal matrix is the sum of $k$ Kronecker products with an unknown configuration. Although it is natural to think of using an information criterion to determine the configuration, we find strikingly that the criterion based on just the mean squared error can identify the configuration, and prove its asymptotic consistency under suitable conditions on the signal-to-noise ratio. We therefore propose a two-step procedure of the stable matrix completion problem, first identify the configuration, and then complete the matrix using the chosen configuration.


For a given configuration of the Kronecker product, if one block of the matrix is completely missing, then it is impossible to recover the entries in it. To alleviate this issue, we also introduce an aggregated procedure which incorporates a few different configurations, and combine the recovered matrices under each of them. This is based on a fairly simple observations that entries non-recoverable under one configuration may be recovered under others, maybe imperfectly but still useful. The aggregation over multiple configurations can also potentially provide more robust and stable estimates of the underlying signal matrix in finite samples under the bias-variance tradeoff framework. We propose an empirical procedure to implement the aggregation procedure, and use a modified cross-validation criterion to select the number of configurations to be used.

The rest of the paper is organized as follows. In Section \ref{sec:matrix-completion-problem}, we formally introduce the matrix completion problem, assuming the underlying complete matrix has a low Kronecker rank. The estimation and configuration determination procedures are presented in Section \ref{sec:method}. The theoretical analysis on the consistency of the configuration selection and the error bound of recovered matrix are provided in Section~\ref{sec:analysis}. We explore the power of aggregation by incorporating different configurations in Section~\ref{sec:averaging}. Section \ref{sec:example} employs simulation studies and a real image example to demonstrate the performances of both the configuration selection and the matrix completion procedures. Section \ref{sec:conclusion} concludes.

\textbf{Notation:} $[K]=\{0,1,\dots, K\}$ denotes the set of non-negative integers less than or equal to $K$. For a vector $u$, $\|u\|$ denotes its Euclidean norm. For a matrix $\bm M\in\mathbb R^{m\times n}$, $\|\bm M\|_F$ represents its Frobenius norm such that $\|\bm M\|_F = \sqrt{\text{tr}(\bm M\bm M')}$ and $\|\bm M\|_S$ denotes its spectral norm defined by $\|\bm M\|_S=\argmax_{u\in \mathbb R^n, \|u\|=1} \|\bm Mu\|$. For positive integer $P$, we denote the set of all divisors of $P$ by $d(P) = \{p\in \mathbb Z^+: P\mod p =0\}$.

\section{Low K-rank Matrix Completion}
\label{sec:matrix-completion-problem}

\subsection{K-rank and the matrix completion}
Let $\bm M\in\mathbb R^{P\times Q}$ be a $P\times Q$ matrix, and suppose
$P$ and $Q$ can be factorized as $P=pp^*$ and $Q=qq^*$. It is shown
\citep{van1993approximation} that $\bm M$ has a
complete Kronecker product decomposition (KPD) in the form
\begin{equation} \label{KPD}
\bm X=\sum_{i=1}^{(pq)\wedge (p^*q^*)}\lambda_i\bm A_i\otimes \bm B_i    
\end{equation}
where $\bm A_i\in\mathbb R^{p\times q}$ and
$\bm B_i\in\mathbb R^{p^*\times q^*}$, and the
Kronecker product of $\bm A\in\mathbb R^{p\times q}$ and
$\bm B\in\mathbb R^{p^*\times q^*}$ is given as
\begin{equation}
\bm A\otimes \bm B :=\begin{bmatrix}
a_{1,1}\bm B & a_{1,2}\bm B &\dots & a_{1,q}\bm B\\
a_{2,1}\bm B & a_{2,2}\bm B &\dots & a_{2,q}\bm B\\
\vdots &\vdots & \ddots &\vdots\\
a_{p,1}\bm B & a_{p,2}\bm B &\dots & a_{p,q}\bm B
\end{bmatrix},\label{eq:kronecker-product}
\end{equation}
where $a_{i, j}$ is the element of $\bm A$ in $i$-th row and $j$-th column.
In \eqref{KPD} we require
\[
\lambda_1 \geqslant \lambda_2 \geqslant \dots \geqslant \lambda_{(pq)\wedge (p^*q^*)} \geqslant 0,
\]
and
\begin{equation*}
    \mathrm{tr}(\bm A_i\bm A_j^T)=\mathrm{tr}(\bm B_i\bm B_j^T)=\left\{\begin{array}{ll}
    1     & \hbox{when } i=j \\
    0     & \hbox{when } i\neq j
    \end{array}\right.,
\end{equation*}
so that the matrices $\bm A_i$ and $\bm B_i$ are identified up to a sign change when $\lambda_i$ are all distinct. The $(p, q, p^*, q^*)$ or simply $(p,q)$ (when $P$ and $Q$ are given) is called the \textit{configuration} of the Kronecker
product decomposition. When $P$ and $Q$ have multiple integer factors, there are
multiple configurations and therefore a KPD in the form \eqref{KPD} exists for each configuration.

We propose to consider the matrix completion problem when the matrix $\bm X \in \mathbb{R}^{P\times Q}$, which we try to recover, has a low rank KPD such as
\begin{equation} \label{rank_KPD}
\bm X=\sum_{i=1}^{r}\lambda_i\bm A_i\otimes \bm B_i
\end{equation}
with configuration $(p_0, q_0)$ and rank $r\leq (p_0q_0)\wedge (p_0^*q_0^*)$. We refer to $r$ as the Kronecker rank (K-rank) of $\bm X$ with respect to the configuration $(p_0,q_0)$. The relationship between K-rank and KPD is similar to, and an extension of the relationship between rank and singular value decomposition (SVD), which renders $\bm X$ in the form
\begin{equation} \label{rank_SVD}
\bm X=\sum_{i=1}^{r}\lambda_iu_i v_i^T,
\end{equation}
a sum of $r$ rank one matrices. It is seen that the SVD in \eqref{rank_SVD} is the KPD in \eqref{rank_KPD} with the configuration $(P,1,1,Q)$. In fact,
there is a deeper and more intrinsic connection between \eqref{rank_KPD} and \eqref{rank_SVD},
as will be explained in Section~\ref{sec:connection}.

We assume that
  $\bm Y$ is a corrupted version of the underlying signal matrix $\bm X$ such that
\begin{equation}
\bm Y = \bm X +\dfrac{\sigma}{\sqrt{PQ}} \bm E,\label{eq:kronecker-model}
\end{equation}
where $\bm E$ is a $P\times Q$ matrix with IID standard Gaussian entries and $\sigma$ denotes the noise level. 

In the matrix completion problem, the matrix $\bm Y$ is only partially observed. In this paper, we consider the missing completely at random scheme by assuming that the observed matrix $\bm Y^*$ is generated as
\[
  [\bm Y^*]_{ij}=[\bm Y]_{ij}\delta_{ij}
\]
where $\delta_{ij}$ are independent and identically distributed $\sim$ Bernoulli$(\tau)$, which are also independent with $\bm X$ and $\bm E$. The rate $\tau$ is called the \textit{observing rate}. 
We also introduce another notation of the observed matrix through a projection operator. Let $\Omega$ be the set of indices of the observed entries in $\bm Y^*$ and define the projection $P_\Omega(\cdot)$ as
$$[P_\Omega(\bm M)]_{ij} = \begin{cases}
[\bm M]_{ij}& \text{if } (i, j)\in\Omega,\\
0&\text{otherwise},
\end{cases}$$
for any $P\times Q$ matrix $M$. In particular, it follows that $\bm Y^*=P_\Omega(\bm Y)$, where unobserved entries are filled with zeros by default.


\subsection{Connection to low rank matrix completion}\label{sec:connection}
The prevailing matrix completion algorithms \citep{candes2009exact, candes2010matrix, jain2013low}  assume a
low rank structure of the signal matrix $\bm X$ as in \eqref{rank_SVD}. It corresponds to a special low K-rank form of $\bm X$ in \eqref{rank_KPD}, with the configuration $p=P, q=1$ and $p^*=1, q^*=Q$. 
In this section, we demonstrate the idea of \citet{van1993approximation} and \citet{cai2019kronecker} that with a \textit{known} configuration, the matrix completion problem with KPD in \eqref{rank_KPD} can be converted to a standard low rank matrix completion problem and be solved using the existing approaches. 

As discussed in \citet{van1993approximation} and \citet{cai2019kronecker}, the Kronecker product of two matrices and the outer product of their vectorized version are linked through a rearrangement operation. It can be seen from
(\ref{eq:kronecker-product}) that all the entries in $\bm A\otimes \bm B$ have
the form $a_{i,j}b_{k,\ell}$ and all the entries in $\vec(\bm A)\vec(\bm B)^T$ have
the same form $a_{i,j}b_{k,\ell}$, where
$\vec(\cdot)$ is the vectorization operation that flattens a matrix to a column vector. Hence
$\bm A\otimes \bm B$ and $\vec(\bm A)\vec(\bm B)^T$ have exactly the same
entries
in the matrices, except the arrangement -- one is a $pp^*\times qq^*$ matrix
and the other is a $pq\times p^*q^*$ matrix. 

We define a rearrangement
operator $\mathcal R_{p, q}[\cdot]$ so that
$\mathcal R[\bm A\otimes \bm B]=\vec(\bm A)\vec(\bm B)^T$.
Specifically, assuming $\bm M$ is a $pp^*\times qq^*$ matrix, a rearrangement operation $\mathcal R_{p,q}[\cdot]$ is defined as
$$\mathcal R_{p,q}[\bm M]=[\vec(\bm M_{1,1}),\dots, \vec(\bm M_{p, q})]^T,$$
where $\bm M_{i, j}$ is the $(i, j)$-th
block of $\bm M$ of size $p^*\times q^*$. Then
for any $p\times q$ matrix $\bm A$ and $p^*\times q^*$ matrix $\bm B$,
we have
$$\mathcal R_{p,q}[\bm A\otimes \bm B] = \vec(\bm A)[\vec(\bm B)]^T.$$
That is, if the configuration of the rearrangement operator is the same as that of the operand,
the rearrangement operator $\mathcal R_{p,q}$ turns a Kronecker product into a rank-one matrix. Note that
the rearrangement operation is linear, in that $\mathcal R_{p,q}(c_1\bm M_1+c_2\bm M_2)=
c_1\mathcal R_{p,q}(\bm M_1)+c_2\mathcal R_{p,q}(\bm M_2)$ and the inverse operator
$\mathcal R_{p,q}^{-1}$ exist so that $\mathcal R_{p,q}^{-1}(\mathcal R_{p,q}(\bm M))=\bm M$. The 
rearrangement operator also retains the Frobenius norm in that 
$\|\mathcal R_{p,q}(\bm M)\|_F=\|\bm M\|_F$ since both matrices contain the same set of elements.
As a result, the KPD \eqref{rank_KPD} becomes a SVD after the rearrangement,
and many problems pertaining to KPD can therefore be solved under the realm of SVD. 

If we apply $\mathcal R_{p_0, q_0}$ to the corrupted matrix $\bm Y$ in \eqref{eq:kronecker-model}, assuming that $\bm X$ has a low K-rank $r\leq (p_0q_0)\vee(p_0^*q_0^*)$, it returns a corrupted low rank matrix
\begin{equation}
\mathcal R_{p_0, q_0} [\bm Y] = \sum_{i=1}^r \lambda_i\vec(\bm A_i) [\vec(\bm B_i)]^T + \dfrac{\sigma}{\sqrt{PQ}} \mathcal R_{p_0, q_0}[\bm E],\label{eq:equivalent-low-rank}
\end{equation}
where $\mathcal R_{p_0, q_0}[\bm E]$, the rearranged noise matrix, is still a matrix of IID Gaussian entries. Note that $\mathcal R_{p_0, q_0}[P_\Omega \bm Y] = P_{\bar\Omega_{p_0, q_0}}\mathcal R_{p_0, q_0}[\bm Y]$ where $\bar\Omega_{p_0, q_0}$ denotes the indices of the observed entries after rearrangement. The K-rank-$r$ matrix completion problem based on $P_\Omega \bm Y$ is therefore equivalent to the standard rank-$r$ matrix completion problem based on $P_{\bar\Omega_{p_0, q_0}}\mathcal R_{p_0, q_0}[\bm Y]$. 

The preceding discussion assumes that the true configuration $(p_0, q_0)$ is known so that the rearrangement operator in \eqref{eq:equivalent-low-rank} is determined. However, observing only the matrix $P_\Omega\bm Y$ and its dimension $(P, Q)$ does not reveal the true configuration $(p_0, q_0)$ --- any configuration $(p, q)$ satisfying $p\in d(P)$ and $q\in d(Q)$ might be the true configuration, where $d(P)$ is the set of factors (divisors) of $P$. If the rearrangement operator $\mathcal R$ is erroneously configured, the low rank structure in \eqref{eq:equivalent-low-rank} might be violated since $\mathcal R_{p, q}[\bm A\otimes \bm B]$ typically does not preserve the low rank structure if $(p, q)\neq(p_0, q_0)$. 



\section{Methodology}\label{sec:method}
In this section, we propose a two-step procedure to solve the matrix completion problem under the assumption
that the signal matrix $\bm X$ is of low K-rank under a proper configuration. In the first step, we determine the configuration by maximizing a criterion function over all candidate configurations. Specifically, we use the spectral norm of the rearranged version of $P_\Omega[\bm Y]$ as the criterion function. In the second step, we convert the observed matrix to a standard matrix completion problem as in \eqref{eq:equivalent-low-rank} with respect to the estimated configuration from first step. The rearranged partially observed matrix is then completed by an alternating least squares (ALS) algorithm.


\subsection{Configuration Determination}\label{sec:method-config}
Recall that $(P,Q)$ is
the dimension of $\bm Y^*=P_\Omega\bm Y$.
We introduce the following candidate set
\begin{equation}
\label{eq:candidate}
    \mathcal C_\delta =\{(p, q): p\in d(P), q\in d(Q), (PQ)^{1/4+\delta}\leqslant pq\leqslant (PQ)^{3/4-\delta}\},
\end{equation}
where $0<\delta<1/4$ is a positive constant. Again, $d(P)$ is the set of all factors of $P$. For each configuration $(p,q)$ in the set $\mathcal C_\delta$, the rearranged $\mathcal R_{p,q}[P_\Omega \bm Y]$ has an aspect ratio between $(PQ)^{-1/2+2\delta}$ and $(PQ)^{1/2-2\delta}$. Here, extreme aspect ratios are excluded for several reasons. First, the extreme ``corner" cases where $pq$ is extremely small or large are of less interest. For example $p=q=1$ and $p=P, q=Q$ represent the case of a product of scalar and the full matrix. 
When $p=q=2$, the signal matrix $\bm X$ is assumed to be a block matrix formed by two-by-two matrices (total $PQ/4$
of them) of the
same proportion, while when $p=P/2, q=Q/2$, the matrix $X$ is a two-by-two block matrix of size 
$(P/2)\times (Q/2)$. The model complexity is high and the
dimension reduction in these cases is limited, similar to using a very small bandwidth in kernel smoothing operations.
Second, from the view of model selection using information criteria, we can view $\delta$ as a model complexity penalty parameter with the penalty function
$pen(p,q)=0$ if $(p,q)\in \mathcal C_\delta$ and $pen(p,q)=\infty$ otherwise. Third, we expect the partially observed matrix $\mathcal R_{p_0, q_0}[P_\Omega \bm Y]$ to be recoverable under the true configuration while an extreme aspect ratio usually results in a non-negligible probability of missing an entire row or column, in which case these missing items are not recoverable. The 
recoverability issue will be discussed in detail later. Lastly, the aspect ratio is a crucial factor in quantifying the error of standard low rank matrix completion \citep{keshavan2010matrixcompletion, jain2013low, gunasekar2013noisy}. Theoretical performance guarantee exists only when the aspect ratio is in a reasonable range.

We propose to determine the configuration through the following maximization problem,
\begin{equation}
    (\hat p, \hat q) = \argmax_{(p, q)\in\mathcal C_\delta}\ \|\mathcal R_{p, q}[P_\Omega \bm Y]\|_S.\label{eq:config-maximization}
\end{equation}
Recall that in $P_\Omega \bm Y$, unobserved entries are filled in with zeros.

We now briefly discuss the heuristics behind the procedure \eqref{eq:config-maximization}. The rearranged partially observed matrix $\mathcal R_{p, q}[P_\Omega \bm Y]$ can be viewed as a partially observed version of the rearranged matrix $P_{\bar\Omega_{p, q}} \mathcal R_{p, q}[\bm Y]$. Its spectral norm is therefore
\begin{equation}
\|\mathcal R_{p, q}[P_\Omega \bm Y]\|_S=\|P_{\bar\Omega_{p, q}} \mathcal R_{p, q}[\bm Y]\|_S\approx \tau \|\mathcal R_{p, q}[\bm X]\|_S.\label{eq:spectral-approx}
\end{equation}
As discussed in \citet{cai2019kronecker}, under mild conditions, $\|\mathcal R_{p, q}[\bm X]\|_S$ attains its maximum over $\mathcal C_\delta$ at the true configuration. The error in the approximation \eqref{eq:spectral-approx} comes from two sources: the noise matrix $\bm E$ and the missing entries.
Both can be controlled when we operate within the candidate set $\mathcal C_\delta$ for $\delta>0$. The analysis is presented in Section~\ref{sec:analysis}.

\subsection{Estimation}\label{sec:method-estimation}

Once the configuration is determined, the Kronecker matrix completion problem under the 
configuration $(\hat p, \hat q)$ becomes a standard low rank matrix completion problem on the 
rearranged matrix and the its rank is the same as the K-rank of the original matrix. Existing procedures such as that proposed by \cite{ashraphijuo2017rank} for identifying the rank are readily applicable. With a fixed

\begin{align}
  \min_{\lambda_i, \bm A_i, \bm B_i}\  \left\|P_{\Omega}\left(\sum_{i=1}^r\lambda_i\bm A_i\otimes \bm B_i\right) - P_{\Omega}\bm Y\right\|_F. \label{eq:optimization}
\end{align}
where $\bm A_i\in\mathbb R^{\hat p\times \hat q}$, $\bm B_i\in\mathbb R^{\hat p^*\times \hat q^*}$ and $\|\bm A_i\|_F = \|\bm B_i\|_F=1$ with $\hat p^*=P/\hat p$ and $\hat q^*=Q/\hat q$. 
Since the rearrangement operation preserves the Frobenius norm, the optimization problem \eqref{eq:optimization} is equivalent to the classical rank-$r$ matrix completion problem
\begin{equation}
      \min_{\lambda_i, u_i, v_i}
        \left\|P_{\bar\Omega_{\hat p,\hat q}}\left(\displaystyle\sum_{i=1}^r\lambda_i u_iv_i^T-\mathcal R_{\hat p,\hat q}[\bm Y]\right)
        \right\|_F.\label{eq:optimization-rearranged}
\end{equation}
where $u_i=\vec(\bm A_i)$, $v_i = \vec(\bm B_i)$, and $\bar\Omega_{\hat p, \hat q}$ records the indices of observed entries after the rearrangement.
To solve the optimization in (\ref{eq:optimization-rearranged}), we adopt the alternating least squares (ALS) algorithm proposed by \citet{jain2013low}, where the initial values for $u_i$ and $v_i$ are directly estimated from the singular value decomposition of $P_{\bar\Omega_{\hat p, \hat q}}\mathcal R_{\hat p,\hat q}[\bm Y]$ as
in \citet{keshavan2010matrix}. The algorithm is depicted in Algorithm \ref{alg:alternating-minimization}. The recovered matrix is therefore $\hat{\bm X}=\mathcal R^{-1}_{\hat p,\hat q}[\sum_{i=1}^r\hat\lambda_i \hat u_i \hat v_i^T]$, where $\mathcal R_{\hat p,\hat q}^{-1}$ is the inverse operator of the rearrangement and $\hat \lambda_i$, $\hat u_i$ and $\hat v_i$ are the optimal solution of \eqref{eq:optimization-rearranged}.

\begin{algorithm}[!tb]
    \caption{Matrix Completion with Fixed Configuration}
    \label{alg:alternating-minimization}
    \begin{algorithmic}[1]
        \STATE {\bfseries Input:} matrix $P_\Omega\bm Y$, configuration $(p, q)$.
        \STATE Select rank $r$. 
        \STATE Let $\tilde{\bm Y}=P_{\bar\Omega_{p, q}}\mathcal R_{p, q}[\bm Y]$.
        \STATE Initialize $\bm\Lambda^{(0)}$, $\bm U^{(0)}$, $\bm V^{(0)}$ such that
        $\bm U^{(0)}\bm \Lambda^{(0)}[\bm V^{(0)}]^T$ is the leading rank $r$ SVD of $\tilde{\bm Y}$.
        \REPEAT
        \STATE $\bm U^{*}=\displaystyle\argmin_{\bm U}\ \|P_{\bar\Omega_{p, q}}[\tilde{\bm Y} - \bm U[\bm V^{(k)}]^T]\|_F$
        \STATE $\bm V^{*}=\displaystyle\argmin_{\bm V}\ \|P_{\bar\Omega_{p, q}}[\tilde{\bm Y} - \bm U^{*}[\bm V]^T]\|_F$
        \STATE Update $\bm \Lambda^{(k+1)}$, $\bm U^{(k+1)}$, $\bm V^{(k+1)}$ such that $\bm U^{*}[\bm V^{*}]^T=\bm U^{(k+1)}\bm\Lambda^{(k+1)}[\bm V^{(k+1)}]^T$ is in standard SVD form.
        \UNTIL{convergence}
        \STATE    Return $\hat{\bm X}=\mathcal R^{-1}_{p,q}[\bm U\bm\Lambda\bm V^T]$ where $\bm \Lambda, \bm U, \bm V$ are from the last iteration.
    \end{algorithmic}
\end{algorithm}

\section{Theoretical Analysis}\label{sec:analysis}


In this section, we present the theoretical properties of our methodology when the signal $\bm X$ is of
K-rank 1. That is, 
\begin{equation}
\label{eq:krank_1}
\bm Y = \lambda \bm A\otimes \bm B + \dfrac{\sigma}{\sqrt{PQ}} \bm E,
\end{equation}
where $\bm A\in\mathbb R^{p_0\times ,q_0}$, $\bm B\in\mathbb R^{p_0^*\times q_0^*}$,  $\|\bm A\|_F=\|\bm B\|_F=1$ and $\bm E$ is a Gaussian ensemble matrix with IID standard normal entries. 
Here we normalize the error matrix by $(PQ)^{-1/2}$ such that the signal-to-noise ratio of the fully observed $\bm Y$ is $\lambda^2/\sigma^2$. 
The analysis in this section can be extended to $r>1$ cases without substantial differences. 
The presentation for $r=1$ case is cleaner and easier to understand.

\subsection{Assumptions}
In the subsequent analysis, we consider the following asymptotic scheme.
\begin{assump}[{\bf True configuration}]\label{assump:true-config}
Assume the dimension of $\bm Y$ goes to infinity, i.e. $PQ\rightarrow \infty$,
and there is an absolute constant $0<\delta < 1/4$ such that the true configuration $(p_0, q_0)$ satisfies
$$(PQ)^{1/4+\delta} \leqslant p_0q_0 \leqslant (PQ)^{3/4-\delta}.$$
\end{assump}
The requirement on the true configuration in Assumption~\ref{assump:true-config} can be equivalently stated as $(p_0, q_0)\in\mathcal C_\delta$. Note that we do not require both $P\rightarrow \infty$ and $Q\rightarrow \infty$. We only require that jointly $PQ\rightarrow \infty$. In addition, $\delta>0$ implies that the dimensions $p_0q_0\rightarrow \infty$ and $p_0^*q_0^*\rightarrow \infty$, or for both component matrices $\bm A$ and 
$\bm B$, at least one of its dimensions goes to infinity. In the meantime, as discussed in Section~\ref{sec:method-config}, the aspect ratio of $\mathcal R_{p_0, q_0}[\bm Y]$ is restricted between $(PQ)^{-1/2+2\delta}$ and $(PQ)^{1/2-2\delta}$. 

\begin{defi}[Incoherence Coefficient]\label{def:incoherence}
The incoherence parameter $\mu$ of a $m\times n$ matrix $M$ is defined by
$$\mu=\dfrac{\sqrt{mn}\|\bm M\|_{\max}}{\|\bm A\|_F},$$
where $\|\cdot\|_{\max}$ denotes the maximum absolute entry of a matrix.
\end{defi}

\begin{assump}[{\bf Incoherence condition}]\label{assump:incoherence}
Let $\mu$ be the maximum of the incoherence parameters of matrices $\bm A$ and $\bm B$ as defined in Definition~\ref{def:incoherence}.
We assume as $PQ\rightarrow\infty$
$$\lim_{PQ\rightarrow\infty}\  \dfrac{\mu^2}{(PQ)^\delta}=0.$$
\end{assump}

The incoherence condition is commonly imposed in theoretical investigations of matrix completion problems \citep{candes2009exact, candes2010matrix, keshavan2010matrix}.  More specifically, Assumption~\ref{assump:incoherence} translates into an incoherence condition on $\vec(\bm A)$ and $\vec(\bm B)$ , the singular vectors of the rearranged matrix, which is usually assumed for the standard low rank matrix completion problem on $\mathcal R_{p_0, q_0}[P_\Omega \bm Y]$. 

The incoherence parameter $\mu$ is greater or equal to $1$.
One would expect $\mu$ to increase slowly as $PQ\rightarrow\infty$. In the case when a matrix is generated by normalizing a $p\times q$ matrix of IID standard Gaussian entries, its incoherence parameter would be $O_p(\sqrt{\log pq})$ as $pq\rightarrow\infty$. Therefore, if both $\bm A$ and $\bm B$ are generated from random Gaussian matrices (see for example, the random scheme introduced in \citet{cai2019kronecker}), we have $\mu = O_p(\sqrt{\log (p_0q_0\vee p_0^* q_0^*)}) = O_p(\sqrt{(3/4-\delta)\log PQ})$ as $PQ\rightarrow\infty$. 

\begin{assump}[{\bf Observing rate}]\label{assump:observing-rate}
Assume the observing rate $\tau$ satisfies
$$\lim_{PQ\rightarrow\infty}\dfrac{\tau(PQ)^{2\delta}}{(\mu^4+\log PQ)}\rightarrow \infty.$$
\end{assump}
Assumption~\ref{assump:observing-rate} specifies the minimum allowable observing rate $\tau$.
Under Assumption 2, the observing rate $\tau$ is allowed to go to zero, with a decay rate depends on $\delta$ and $\mu$.

Following \citet{cai2019kronecker}, we define 
$$\phi = \max_{(p, q)\in\mathcal W_\delta} \|\mathcal R_{p, q}[\bm A\otimes \bm B]\|_S,$$
where $\mathcal W_\delta := \mathcal C_\delta \setminus \{(p_0, q_0)\}$ is set of incorrect configurations in the candidate set. Since under the true configuration, $\|\mathcal R_{p_0, q_0}[\bm A\otimes \bm B]\|_S=\|\vec(\bm A)\vec(\bm B)^T\|_S=\|\vec(\bm A)\|_F\|\vec(\bm B)^T\|_F=1$, $\phi^2$ can be viewed as the maximum portion (in Frobenius norm) of the matrix $\bm A\otimes \bm B$ that can be represented by a Kronecker product of a wrong configuration. 
On the other hand, for any given configuration $(p, q)\in\mathcal W_\delta$, we always have $\|\mathcal R_{p,q}[\bm A\otimes \bm B]\|_S\leqslant \|\mathcal R_{p,q}[\bm A\otimes \bm B]\|_F=\|\bm A\otimes \bm B\|_F=1$, resulting in $\phi\leqslant 1$.
The \textit{representation gap} $\psi^2$ is therefore defined as
\begin{equation}
\psi^2:= 1- \phi^2.\label{eq:representation-gap}
\end{equation}
The representation gap $\psi^2$, taking value on $[0, 1]$, quantifies the gap between any wrong configuration and the true configuration in the sense of representing the signal part with a K-rank 1 matrix.
If $\psi^2=0$, there exists another configuration $(p', q')\neq (p_0, q_0)$ such that $\bm A\otimes \bm B = \bm A'\otimes \bm B'$ for some $\bm A'\in\mathbb R^{p'\times q'}$ and $\bm B'\in\mathbb R^{P/p'\times Q/q'}$ , hence our model cannot be uniquely determined. A larger value of the representation gap $\psi^2$ makes it easier to separate the true configuration from the wrong ones. As discussed in 
\cite{cai2019kronecker}, this gap is similar to the smallest eigenvalue in determining matrix ranks, or the 
smallest standardized non-zero coefficient in variable selection in regression analysis.

\begin{assump}[SNR and representation gap]\label{assump:snr}
Assume $\psi^2$ and $\lambda/\sigma$ satisfy
\begin{align*}
\lim_{PQ\rightarrow\infty} &\dfrac{\tau\psi^4(PQ)^{2\delta}}{\mu^4}\rightarrow\infty,\mbox{\ \ and \ \ } 
\lim_{PQ\rightarrow\infty} \dfrac{\tau(\lambda/\sigma)^2 \psi^4(PQ)^{2\delta}}{{\log PQ}}\rightarrow\infty.
\end{align*}
\end{assump}
Assumption~\ref{assump:snr} gives the lower bounds on the representation gap $\psi^2$ and the signal-to-noise ratio $\lambda/\sigma$, in terms of the incoherence parameter $\mu$ and the observing rate $\tau$. Comparing with Assumption~\ref{assump:observing-rate} on the observing rate $\tau$, we see that constant values of $\lambda/\sigma$ and $\psi^2$ suffice for Assumption~\ref{assump:snr}, though they may shrink to 0 with a rate depending on $\delta$ and $\tau$, as $PQ\rightarrow \infty$.

\subsection{Properties of the estimated configuration}
To ease the notation, we denote the criterion function by
$$G(p, q) = \|\mathcal R_{p, q}[P_\Omega\bm Y]\|_S,$$
for $(p, q) \in \mathcal C_\delta$. 
Recall that we use this function as the configuration selection criterion in \eqref{eq:config-maximization}. Theorem~\ref{thm:gap} provides the gap in the expectation of $G(p, q)$ between the true configuration and any wrong configurations. 

\begin{thm}[{\bf Gap in criterion function}]\label{thm:gap}
Under Assumptions~\ref{assump:true-config} to 
\ref{assump:snr}, we have
$$\mathbb E[G(p_0, q_0)] - \max_{(p, q)\in\mathcal W_\delta} \mathbb E[G(p, q)] \geqslant \tau\lambda(1-\phi)(1+o(1)),$$
where the expectation is taken on both random observations of entries and the random noise.
\end{thm}
Since $\mathbb E[G(p_0, q_0)]\approx \tau\lambda$, Theorem~\ref{thm:gap} shows that the representation gap of the signal part is preserved using the criterion function $G(p, q)$. 
The proof of the theorem can be found in the Appendix.

On the face of it, Theorem~\ref{thm:gap} is similar to the one for the Kronecker product model introduced in \citet{cai2019kronecker}. However, the underlying matrix completion problem makes it substantively different, in the form and even more so in the proof. We mention two major differences here. First of all, in the matrix completion problem, the candidate set $\mathcal C_\delta$ has to be smaller than the one considered in \citet{cai2019kronecker}, where in the latter case, only two extreme configurations $(0, 0), (P, Q)$ are excluded from $\mathcal C$. The model complexity is controlled by excluding extreme configurations.
From the perspective of the aspect ratio of $\mathcal R_{p, q}[P_\Omega\bm Y]$, the aspect ratio is often controlled as a constant as matrix size increases to infinity in existing literature on conventional matrix completion problems \citep{candes2009exact, candes2010matrix, jain2013low}. Within the configuration set $\mathcal C_\delta$, the aspect ratio is allowed to shrink to 0 or increases to infinity with a controlled rate between $(PQ)^{-1/2+2\delta}$ and $(PQ)^{1/2-2\delta}$. Configurations that result in the aspect ratios of $R_{p, q}[P_\Omega\bm Y]$ outside this range are excluded. 
Second, our assumption on the signal-to-noise ratio and representation gap (Assumption~\ref{assump:snr}) is stronger than its counterpart in \citet{cai2019kronecker}. The reason is that the missingness contributes more to the fluctuation of $G(p, q)$ than the noise matrix $\bm E$ itself (see equation \eqref{eq:proof-noise-bound} in the proof). 

\begin{thm}[{\bf Consistency}]\label{thm:consistency}
 Let $(\hat p, \hat q) = \argmax_{(p, q)\in\mathcal C_\delta}\ G(p, q)$. Under Assumptions~\ref{assump:true-config} to \ref{assump:snr}, we have
$$P[(\hat p , \hat q) = (p_0, q_0)]\geqslant 1 - C\dfrac{|d(P)||d(Q)|}{(PQ)^{3/4+3\delta}},$$
for some constant $C$, where $|d(P)|$ is the size of the set $d(P)$.
\end{thm}

\textbf{Remark:} Theorem~\ref{thm:consistency} is established for the model with K-rank 1. 
If the K-rank is higher than 1, the same procedure \eqref{eq:config-maximization} can still select the true configuration consistently, under a larger signal-to-noise, and an additional assumption on the relative strengths of the terms in \eqref{rank_KPD}. On the other hand, when the K-rank is greater than 1, it is also of interest to consider the joint selection of the configuration and the K-rank. Since it exceeds the scope of this paper, we leave it to the future work.

\textbf{Remark:} According to the number theory \citep{hardy1979introduction}, for any $\epsilon > 0$, $|d(n)|=o(n^\epsilon)$. Therefore the probability of consistency in Theorem~\ref{thm:consistency} converges to 1 as $PQ\rightarrow\infty$. We note that the convergence rate, as presented in Theorem~\ref{thm:consistency}, is only related to $PQ$ and $\delta$ but not the signal-to-noise ratio $\lambda/\sigma$ or the representation gap $\psi^2$, different from that for configuration determination for matrix denoising based on a fully observed matrix \citep{cai2019kronecker}. Again, this is because of the smaller candidate configuration set controlled by $\delta$. In fact, according to the proof of Theorem~\ref{thm:consistency} shown in Appendix, there is a term $\exp\{-C(\lambda/\sigma)^2\psi^4PQ\}$ in the lower bound of the probability of correct determination. 
However, this term is dominated by the main term. 

\subsection{Property of the recovered matrix}

After the configuration has been determined by maximizing the criterion function $G(p, q)$, the second step is to recover the original matrix using Algorithm~\ref{alg:alternating-minimization}. It is known that a completely missed row or column cannot be recovered by a low rank matrix completion. There is a similar phenomenon under the model \eqref{rank_KPD}. For example, if the true configuration $(p_0,q_0)$ is used in Algorithm~\ref{alg:alternating-minimization}, and if a block of $\bm Y$ is completely missing, then the corresponding element of $\bm A_i$ is not recoverable, and neither is the whole missing block. The top row of Figure~\ref{fig:locations} illustrates the situation: the first block (highlighted in black in the upper left panel) is completely missing, and becomes a completely missed row (in the upper right panel) after rearrangement. Therefore, the first element of $\bm A_i$ in \eqref{rank_KPD}, or the first entry of $u_i$ in \eqref{rank_SVD} is not recoverable. 
We provide a formal characterization of irrecoverable entries in Definition~\ref{def:irrecoverable}. 

\begin{defi}[{\bf Irrecoverable entries}]\label{def:irrecoverable}
The entry $\bm X_{i,j}$ is irrecoverable with respect to configuration $(p, q)$ under the observation set $\Omega$ if
either the entire row or the entire column of $\mathcal R_{p, q}[P_\Omega\bm Y]$ that contains $\bm X_{i,j}$ is completely missing.\\
Equivalently, using the notation of index set $\Omega$, the entry $\bm X_{i,j}$ is irrecoverable if either
\begin{equation}
    \left\{(i',j')\in\Omega: \left\lfloor\frac{i-1}{p^*}\right\rfloor=\left\lfloor\frac{i'-1}{p^*}\right\rfloor,\  \left\lfloor\frac{j-1}{q^*}\right\rfloor=\left\lfloor\frac{j'-1}{q^*}\right\rfloor\right\}=\emptyset, \label{eq:irrecoverable-row}
\end{equation}
or
\begin{equation}
    \left\{(i',j')\in\Omega: i'=i+lp^*,\ j'=j+mq^*\text{ for some }l,m\in\mathbb Z\right\}=\emptyset,\label{eq:irrecoverable-col}
\end{equation}
where $\lfloor x\rfloor$ denotes the largest integer that is no greater than $x$.
\end{defi}

\noindent\textbf{Remark:} In Definition~\ref{def:irrecoverable}, \eqref{eq:irrecoverable-row} means 
the row in $\mathcal R_{p, q}[P_\Omega \bm Y]$ that contains $\bm X_{i,j}$ is completely missing, and \eqref{eq:irrecoverable-col} implies that the entire column in $\mathcal R_{p, q}[P_\Omega \bm Y]$ that contains $\bm X_{i,j}$ is missing.

\begin{lem}[{\bf Recoverability}]\label{lem:recoverable}
Under Assumptions~\ref{assump:true-config} and \ref{assump:observing-rate}, as $PQ\rightarrow\infty$, 
$$P[\text{all entries in $P_\Omega\bm Y$ are recoverable under $(p_0, q_0)$}]\geqslant 1 - 2\exp\{-C(PQ)^{1/4-\delta}\log PQ\},$$
for some constant $C$. 
\end{lem}

Lemma~\ref{lem:recoverable} shows that as long as $(p_0, q_0)\in\mathcal C_\delta$, with high probability, there is no missing columns or missing rows in the rearranged matrix $\mathcal R_{p_0, q_0}[P_\Omega \bm Y]$ such that the optimization in \eqref{eq:optimization-rearranged} is well defined. 

\begin{thm}[Recovery error]\label{thm:recovery-error}
Let $\hat {\bm X} = \hat\lambda \hat{\bm A}\otimes \hat{\bm B}$ be the recovered matrix of optimization \eqref{eq:optimization-rearranged} under the true configuration $(p_0, q_0)$ using Algorithm~\ref{alg:alternating-minimization}. Then under Assumption~\ref{assump:incoherence}, with high probability we have
$$
\|\hat{\bm X} - \bm X\|_F\leqslant \lambda \sqrt{PQ}\exp\left\{-\dfrac{C_1\tau(p_0q_0\wedge p_0^*q_0^*)}{\mu^4\log PQ}\right\}
+ C_2\dfrac{\sigma\mu}{\sqrt{\tau}}\sqrt{\log PQ}\cdot (PQ)^{-1/4}(p_0q_0\wedge p_0^*q_0^*)^{-1}
$$
for some global constant $C_1, C_2>0$. 
\end{thm}

Theorem~\ref{thm:recovery-error} gives the error bound for the recovered matrix under the true configurations. Since in Assumption~\ref{assump:true-config} $p_0q_0\wedge p_0^*q_0^* \geqslant (PQ)^{1/4+\delta}$, the error bound in Theorem~\ref{thm:recovery-error} can be revised to 
\begin{equation}
\|\hat{\bm X}-\bm X\|_F\leqslant \lambda \exp\left\{-\dfrac{C_3\tau (PQ)^{1/4+\delta}}{\mu^4\log PQ}\right\} + C_2\dfrac{\sigma\mu}{\sqrt{\tau}}\sqrt{\log PQ}\cdot (PQ)^{-1/2-\delta}.\label{eq:recovery-error-bound}
\end{equation}

Under Assumption~\ref{assump:observing-rate}, one can verify both terms on the right hand side converges to 0. The second term in \eqref{eq:recovery-error-bound} dominates. 
Using $\tau\gg \mu^4(PQ)^{-2\delta}$ from Assumption~\ref{assump:observing-rate} and $\mu = o((PQ)^{\delta/2}$ from Assumption~\ref{assump:incoherence}, one can further have 
$$\|\hat{\bm X} - \bm X\|_F = o_p\left(\sigma \sqrt{\log PQ}(PQ)^{-1/2}\right).$$


\section{Aggregated Estimation}\label{sec:averaging}

Instead of selecting the single best configuration from the candidate set, an alternative is
to incorporate several configurations and combine the recovered matrices under these configurations, leading to a aggregated recovery procedure. One immediate motivation and advantage of this approach is that it alleviates the infeasibility problem when only one configuration is used. Note that if a block is completely missing under the selected configuration $(\hat p,\hat q)$, then all the missing entries in this block cannot be recovered (see  Definition~\ref{def:irrecoverable}). 

We explain heuristically how aggregation can help to alleviate this issue. Let $(p_0, q_0)$ be the true configuration. Suppose the first $P/p_0\times Q/q_0$ block of $\bm Y$ is completely unobserved as in Figure~\ref{fig:locations-origin}, where missing entries in the first block are highlighted in black, and other missing entries are marked in gray. After the rearrangement using the true configuration $(p_0,q_0)$, this block becomes a completely unobserved row in Figure~\ref{fig:locations-true-config}.
If we consider using another configuration $(p, q)$ with $p\geqslant p_0$ and $q\geqslant q_0$, there will be more completely missed rows after the rearrangement. For example, Figure~\ref{fig:locations-config-1} shows $\mathcal R_{p, q}[\bm Y]$ with $(p, q)=(2p_0, q)$, which has two missing rows.
On the other hand, under a configuration $(p,q)$ such that either $p<p_0$ or $q<q_0$, there won't be any completely missed rows after rearrangement. Figure~\ref{fig:locations-config-2} depicts the missing entries of $\mathcal R_{p, q}[\bm Y]$ with $(p, q) = (p_0/2, 2q_0)$, under which those pixels in black are recoverable.
There can be multiple configurations under which the completely missing block is recoverable. Aggregate the recovered matrices under such configurations helps to recover the blocks that are completely missing.


\begin{figure}[!htb]
    \centering
    \begin{subfigure}[t]{0.48\textwidth}
    \centering
    \caption{$\Omega$}\label{fig:locations-origin}
    \frame{\includegraphics{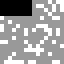}}
    \end{subfigure}
    \begin{subfigure}[t]{0.48\textwidth}
    \centering
    \caption{$\bar\Omega_{p_0, q_0}$}\label{fig:locations-true-config}
    \frame{\includegraphics{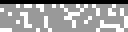}}
    \end{subfigure}\\[1em]
    \begin{subfigure}[t]{0.48\textwidth}
    \centering
    \caption{$\bar\Omega_{2p_0, q_0}$}\label{fig:locations-config-1}
    \frame{\includegraphics{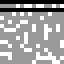}}
    \end{subfigure}
    \begin{subfigure}[t]{0.48\textwidth}
    \centering
    \caption{$\bar\Omega_{p_0/2, 2q_0}$}\label{fig:locations-config-2}
    \frame{\includegraphics{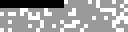}}
    \end{subfigure}
    \caption{Locations of unobserved entries under rearrangements of different configurations, where white pixels represent observed entries, black pixels represent the unobserved entries in the first (upper-left most) block and gray pixels represent other missing entries.}
    \label{fig:locations}
\end{figure}

In this section we aim to provide an empirical procedure of the aggregated estimation, and discuss its properties and advantages. We also propose to use cross validation to select the number of configurations to be incorporated. 

\subsection{Aggregated Estimation}



Let $C_k=(p_{k},q_{k}), k=1,2, \ldots$
be the sequence of configurations ordered according to the criterion function \eqref{eq:config-maximization} within the candidate configuration set $\mathcal C_\delta$,
and let $\hat{\bm X}_k$ be the estimated $\bm X$ using configuration $C_k$. 
For the $(i,j)$-th entry of $\bm X$, and the $k$-th configuration $C_k$, define $\nu_{ijk}$ be the feasibility indicator which takes value 0 if the $(i,j)$-th entry is infeasible under configuration $C_k$, and 1 otherwise.
Let $d_{ij}=\min\{k:\,\nu_{ijk}=1\}$ so $C_{d_{ij}}$
is the best configuration under which $(i,j)$-th entry is feasible.
For a given integer $d>0$, the final estimate of the
$(i,j)$-th entry of ${\bm X}$ is obtained as a weighted average of the $(i,j)$-th entries of $\hat{\bm X}_k$, 
\begin{equation}
\hat{\bm X}[i,j]=
\frac{\sum_{k=1}^{d \vee d_{ij}} w_{k}\nu_{ijk}
\hat{\bm X}_k[i,j]}
{\sum_{k=1}^{d \vee d_{ij}} w_{k}\nu_{ijk}}\label{eq: averaging}
\end{equation}
where $w_{k}$ is the weight assigned to configuration
$C_k$. The simplest choice would be the constant weights such that $w_k\equiv 1$. A more refined approach is to use a set of weights that reflects the accuracy of each configuration. 

For most of the missing entries, the aggregated estimator take the weighted average of the recovered entries among the best $d$ configurations. If an entry $(i,j)$ is infeasible under all the $d$ best configurations $C_1,\ldots, C_d$, the above estimator uses the best configuration $d_{ij}>d$ 
under which $(i,j)$-th entry is feasible, and 
fill it with the recovered entry under that configuration. 

{\bf Remark:} The aggregated estimation \eqref{eq: averaging} can also be viewed as a model averaging procedure. The benefit of the aggregation is multi-fold. First it provides an effective approach to handle the infeasibility issue. The probability that there is an entry that is infeasible under all possible configurations is extremely small, for reasonably large $(M,N)$, and number of factors of 
$P$ and $Q$. Hence the procedure is able to handle higher missing rates. Second, model averaging can potentially provide more robust and stable estimators, as demonstrated in many studies in statistics literature \citep{Buckland1997model, raftery1997bayesian}. Note that for each possible configuration, there is a corresponding
KPD (\ref{KPD}), which, when truncated at a given K-rank $r$, provides an approximation to the signal matrix $\bm X$ in (\ref{eq:kronecker-model}), though different approximations under different configurations and K-ranks are of different qualities. Averaging over the best performing configurations can potentially improve the quality of matrix recovery.  Third, aggregation provides sharper resolution in the completed matrix, particularly in image reconstruction. Kronecker products induce a
block structure in the resulting matrix hence often produce `grainy' images. Averaging over several 
configurations reduces such effects. Fourth,
the final predictive model after averaging is equivalent to a hybrid Kronecker product model  
\[
\bm X=\sum_{k=1}^d w_k{\bm X_k},
\]
where each $\bm X_k$ assumes a KPD of form
 (\ref{rank_KPD}) under configuration
$C_k$. The aggregation approach bypasses the difficulty of jointly estimating such a model as well as  determining the configurations in such a model.
The effectiveness of the approach will be demonstrated in the empirical study.

{\bf Remark:} In some real applications (for example image completion), it is difficult to justify that  $\bm X$ has Kronecker rank 1 and it is possible that $\bm X$ is instead a sum of several Kronecker products with different configurations (see the hybrid decomposition in \cite{cai2019hkopa}). The aggregation procedure is actually identifying those most plausible configurations and recovers the matrix $\bm X$ as a mixture of them as in \eqref{eq: averaging}. 

{\bf Remark:} The empirical study shows that the performance is less sensitive to the number of configurations used in the aggregated recovery \eqref{eq: averaging} if the corner configurations (small or large $m+n$) are excluded. More detailed investigation may be needed.

{\bf Remark:} In practice the matrix dimension $P$ and $Q$ may not have many factors, which limits the flexibility of the KPD approach as the candidate set
$\mathcal C$ can be small. In this case it is possible to augment the
observed matrix with additional missing rows and columns so that the new dimensions $P^*$ and $Q^*$ have more factors. One such choice
is to make $P^*=2^M$ and $Q^*=2^N$. One can also use different $P^*$ and $Q^*$ as part of the aggregating operation.
With a good configuration
determination procedure and effective aggregation, significant improvement in matrix completion tasks
can be obtained. 

\subsection{Choose Number of Terms by Cross-validation}
In practice, it remains to determine the number of different configurations to be incorporated in the aggregation. Here we propose an empirical procedure for determining the number of terms by the $K$-fold cross-validation. Recall $\Omega$ is the set of indices of the observed entries in $\bm Y$. Let $\Omega^{[1]},\cdots, \Omega^{[K]}$ be a random $K$-fold partition of $\Omega$ such that $\Omega^{[1]}\cup\cdots\cup \Omega^{[K]}=\Omega$ and for any $1\leqslant i < j\leqslant K$, we have $\Omega^{[i]}\cap\Omega^{[j]}=\emptyset$ and $||\Omega^{[i]}|-|\Omega^{[j]}||\leqslant 1$. Suppose $\hat {\bm X}(k;\Omega\setminus\Omega^{[i]})$ is the aggregated estimate for $\bm X$ as constructed in \eqref{eq: averaging}, which is obtained by fitting Kronecker matrix completion models on $P_{\Omega\setminus\Omega^{[i]}}\bm Y$ and aggregating the results from the first $k$ configurations in $\mathcal C_\delta$. The $K$-fold cross-validation mean squared error (CV-MSE) is defined by
\begin{align}
    \text{CV-MSE}(k) &= \dfrac{1}{|\Omega|} \left\| P_{\Omega}\bm Y - \sum_{i=1}^K P_{\Omega^{[i]}}\hat{\bm X}(k;\Omega\setminus \Omega^{[i]})\right\|_F^2\nonumber\\
    &=\dfrac{1}{|\Omega|} \left\| \sum_{i=1}^K \left\{P_{\Omega^{[i]}}\bm Y-P_{\Omega^{[i]}}\hat{\bm X}(k;\Omega\setminus \Omega^{[i]})\right\}\right\|_F^2
    .
    \label{eq:cv-mse}
\end{align}
The number of models can be selected by minimizing CV-MSE$(k)$ over $k=1,2,\dots$. In principle, CV-MSE$(k)$ can also be used to estimate the performance $\|\bm Y - \hat{\bm X}(k;\Omega)\|_F^2$, including the unobserved entries as well, which cannot be calculated directly.

In $i$-th cross-validation fold, we use $P_{\Omega\setminus \Omega^{[i]}}\bm Y$ as the training set and use $P_{\Omega^{[i]}}\bm Y$ as the test set. The split of data decreases the observation rate in $P_{\Omega\setminus \Omega^{[i]}}\bm Y$ by $1/K$ and increase the unrecoverability and the number of unrecoverable entries, which may worsen the performance. To minimize the potential impact of the decrease in observing rate, a large value of $K$ is preferred. On the other hand, larger $K$ involves heavier computation. In the numerical analysis that follows, we choose $K=10$ as a compromise. 

\section{Empirical Examples}\label{sec:example}

\subsection{Simulation: Configuration Selection}\label{sec:simulation}
In this simulation experiment, we demonstrate the performance of the configuration determination procedure under the model \eqref{eq:kronecker-model} with a $\bm X$ of K-rank 1. Specifically, the component matrices $\bm A$ and $\bm B$ are randomly generated as
\begin{align*}
    \bm A & = \binom{\sqrt{1-\varphi^2}}{0}\otimes \bm D_1 + \binom{0}{\varphi}\otimes \bm D_2,\\
    \bm B & = \binom{\sqrt{1-\varphi^2}}{0}\otimes \bm D_3 + \binom{0}{\varphi}\otimes \bm D_4,
\end{align*}
where $\|\bm D_i\|_F=1,\,i=1,2,3,4$, $\bm D_1$ and $\bm D_2$ are $2^{m_0-1}\times 2^{n_0}$ matrices such that $\mathrm{tr}(\bm D_1\bm D_2^T)=0$, and $\bm D_3$ and $\bm D_4$ are  $2^{M-m_0-1}\times 2^{N-n_0}$ matrices such that $\mathrm{tr}(\bm D_3\bm D_4^T)=0$. We design such a construction so that the representation gap can be controlled at roughly $\varphi^2$. 
Let $\bm X = \lambda \bm A\otimes \bm B,$
where $\lambda$ is the parameter used to control the signal-to-noise ratio. The underlying complete matrix
$\bm Y$ is generated according to
$\bm Y = \bm X + \sigma 2^{-(M+N)/2}\bm E,$
where $\bm E$ contains IID standard normal entries. The observation set $\Omega$ is sampled independently such that $P[(i, j)\in\Omega] = \tau$ for all $(i, j)$.

We consider two dimension and configuration setups: $M=N=9$ with true configuration $(m_0, n_0)=(4, 4)$, and $M=N=10$ with true configuration $(m_0, n_0)=(5, 4)$. All combinations of
20 different signal-to-noise ratio values $\lambda/\sigma\in\{0.1, 0.2,\dots, 2.0\}$, two different observing rates
$\tau=0.1$ and $0.2$, three different values of $\varphi^2=0.3, 0.4$, and $0.5$ and three different candidate sets $\mathcal C_5, \mathcal C_6$, and $\mathcal C_7$ are considered, where
$$\mathcal C_s := \{(p, q): p\in d(P), q\in d(Q), 2^s\leqslant pq\leqslant PQ 2^{-s}\}.$$
here $s$ is an integer satisfying $(M+N)/4< s<(M+N)/2$. 
With slight misuse of notations, the candidate set $\mathcal C_s$ is equivalent to $\mathcal C_\delta$ in \eqref{eq:candidate} with corresponding value of $\delta=s/(M+N)-1/4$. 
However, it is more convenient and intuitive to use $\mathcal C_s$ here. Note that $\mathcal C_5$ is the largest set among the three.

For each combination of $\lambda/\sigma$, $\tau$, $\varphi^2$ and $\mathcal C$, we repeat the simulation 100 times and record how many times the true configuration
$(m_0, n_0)$ is selected by the criterion in \eqref{eq:config-maximization}. These empirical frequencies are plotted as functions of the signal-to-noise ratio in Figure~\ref{fig:consistency}
for different combinations of $(M, N)$, $\tau$, $\varphi^2$ and $\mathcal C$. We see the general pattern that these functions are all monotone, and when the signal-to-noise ratio is large enough, chances of choosing the correct configuration approach one.

The left panel in Figure~\ref{fig:consistency} reveals that a larger signal-to-noise ratio is required when the candidate set $\mathcal C_s$ is larger (note that $\mathcal C_5$ is the largest among the three of them). This is consistent with Assumption~\ref{assump:snr}, which requires a larger signal-to-noise ratio for a smaller value of $\delta$, i.e. a larger candidate set. 
Theorem~\ref{thm:consistency} also shows that the convergence rate is a decreasing function of $\delta$.
The middle panel of Figure~\ref{fig:consistency} indicates that as the representation gap increases, the probability of making correct choice also increases. 
It is the reason that Assumption~\ref{assump:snr} links the representation gap $\psi^2$ and the required signal-to-noise ratio $\lambda/\sigma$ in a negative relationship. The right panel gives the empirical frequency curves for different dimensions and different observing rates. The pattern is very intuitive, when the dimension is larger, or the observing rate is higher, the performance is better.

\begin{figure}[!htb]
    \centering
    \includegraphics[width=0.32\textwidth]{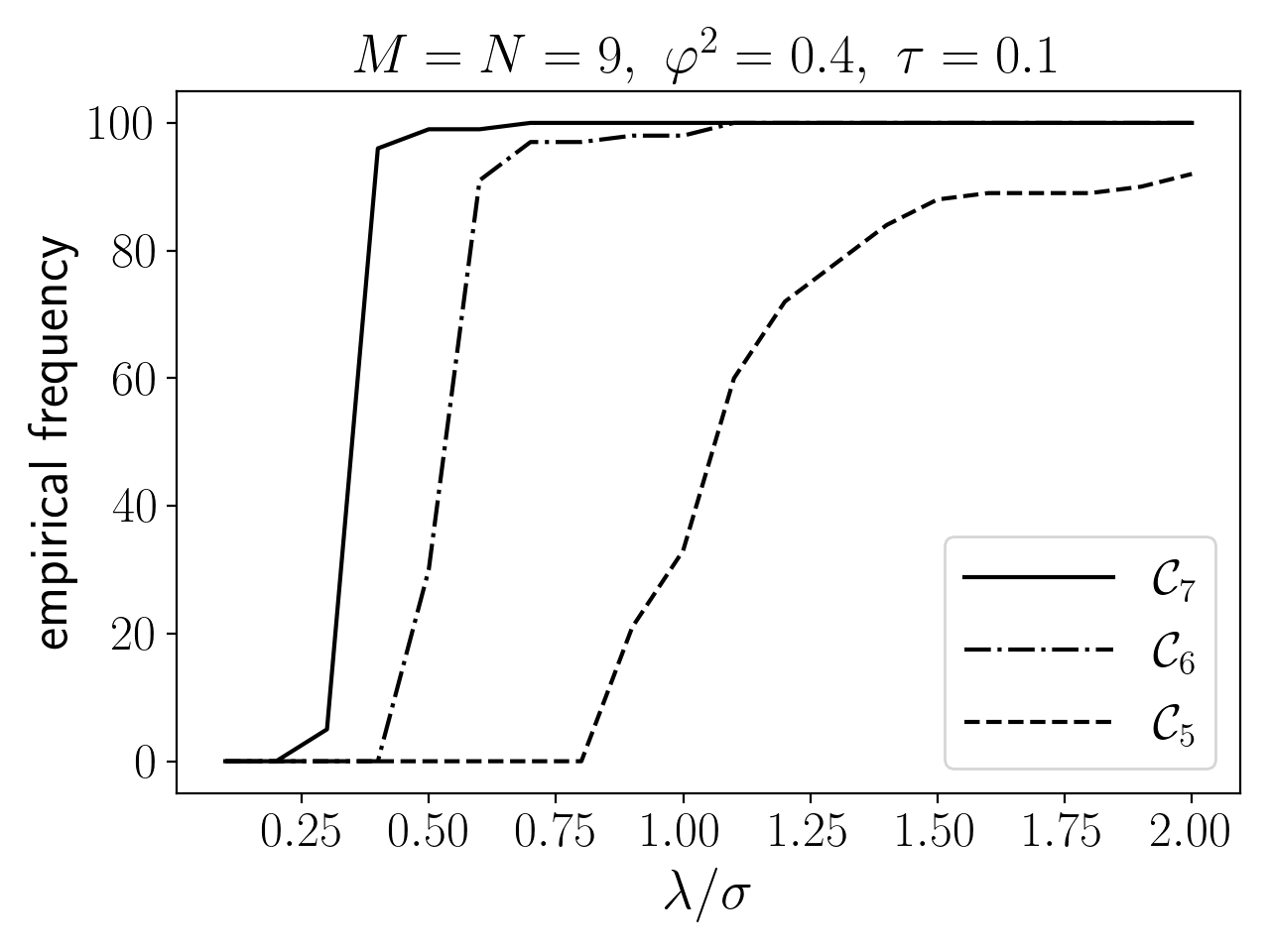}
    \includegraphics[width=0.32\textwidth]{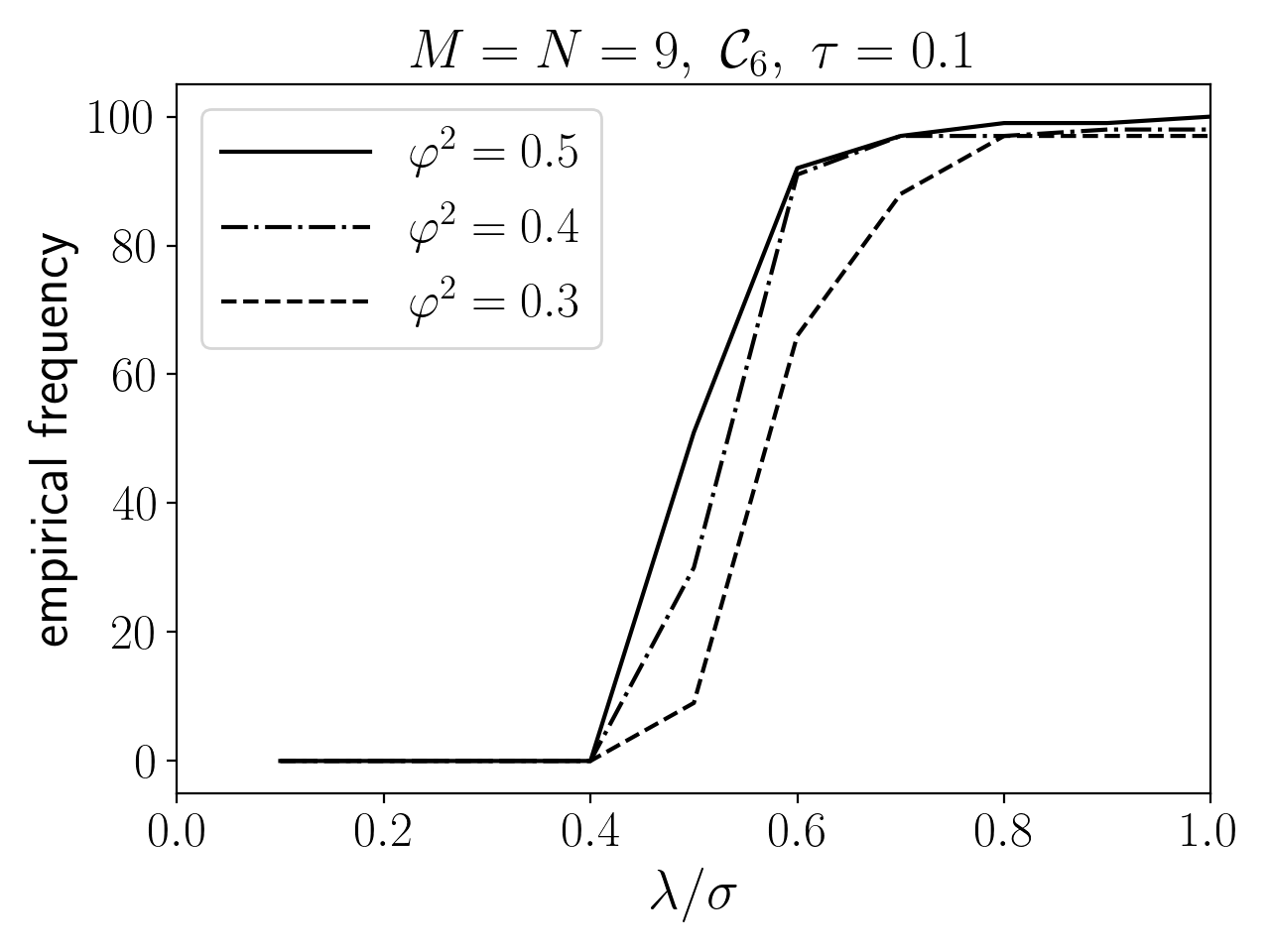}
    \includegraphics[width=0.32\textwidth]{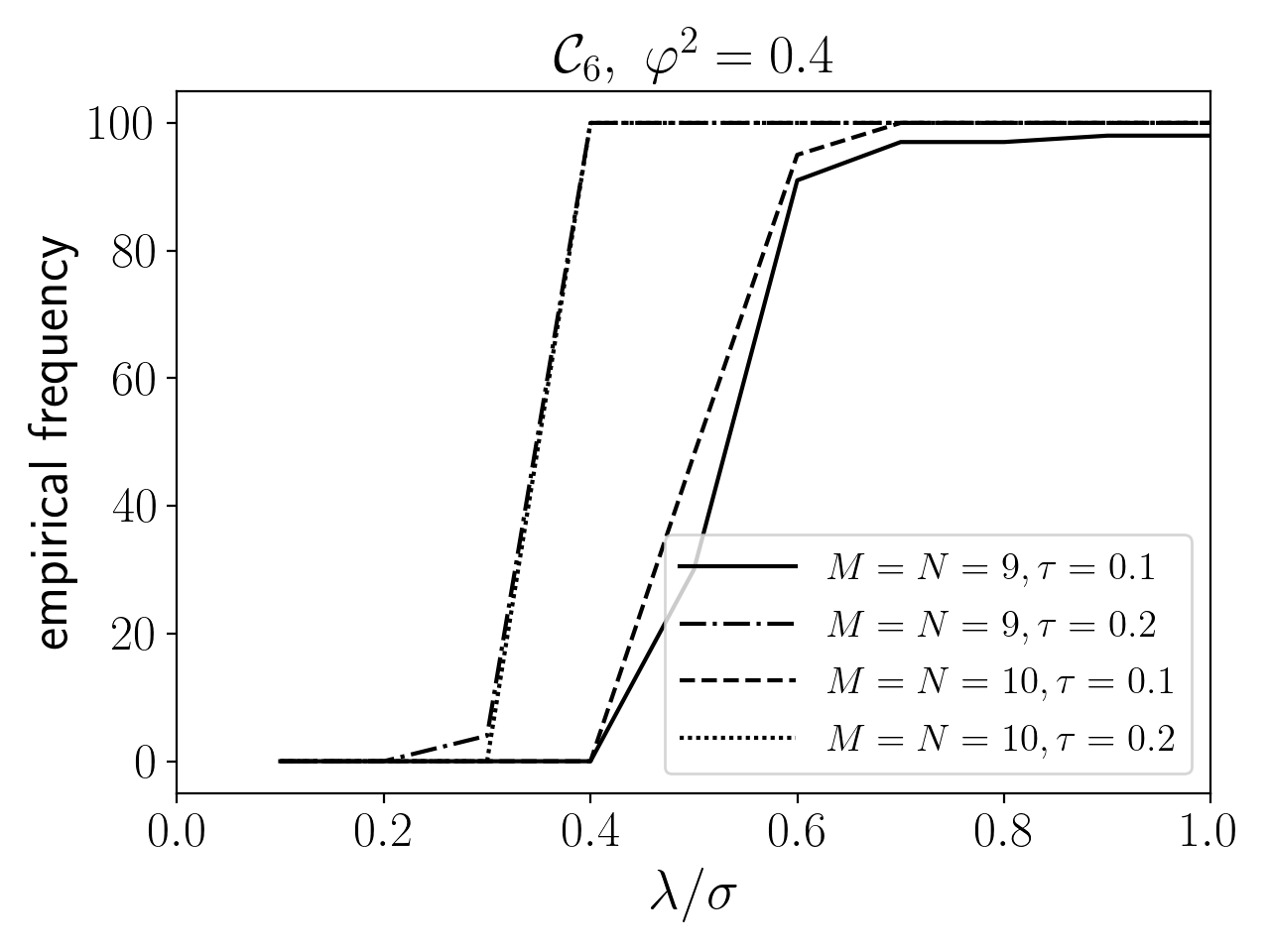}
    \caption{The empirical frequency (in \%) of correct configuration selection increases monotonically as the signal-to-noise ratio for (Left) different candidate sets, (Mid) different representation gaps and (Right) different dimensions and observing rates.}
    \label{fig:consistency}
\end{figure}

To give another comparison, for any combination of $(M, N)$, $\tau$, $\varphi^2$ and $\mathcal C$, we define $\gamma^*$ be the smallest signal-to-noise ratio (truncated to 0.1) such that the empirical frequency of correct configuration selection exceeds 50\%. The results are reported in Table~\ref{tab:transition-point}.

\begin{table}[!htb]
    \centering
    \renewcommand{\arraystretch}{0.9}
    \begin{tabular}{|c|c|ccc|ccc|}
    \hline
     &$\tau$ & \multicolumn{3}{|c}{$0.1$} & \multicolumn{3}{|c|}{$0.2$}\\
     \cline{2-8}
     &$\varphi^2$ & 0.3 & 0.4 & 0.5 & 0.3 & 0.4 & 0.5\\
     \hline
     \multirow{3}{*}{$M=N=9$} & $\mathcal C_5$ & 1.6 & 1.1 & 1.0 & 0.6 & 0.5 & 0.5\\
     & $\mathcal C_6$ & 0.6 & 0.6 & 0.5 & 0.4 & 0.4 & 0.4\\
     & $\mathcal C_7$ & 0.4 & 0.4 & 0.4 & 0.3 & 0.3 & 0.3\\
     \hline
     \multirow{3}{*}{$M=N=10$} & $\mathcal C_5$ & 1.3 & 1.0 & 0.9 & 0.6 & 0.5 & 0.5\\
     & $\mathcal C_6$ & 0.6 & 0.6 & 0.5 & 0.4 & 0.4 & 0.4\\
     & $\mathcal C_7$ & 0.4 & 0.4 & 0.4 & 0.3 & 0.3 & 0.3\\
     \hline
    \end{tabular}
    \caption{Value of $\gamma^*$ for different combinations of $(M, N)$, $\tau$, $\psi^2$ and $\mathcal C$.}
    \label{tab:transition-point}
\end{table}

\subsection{Simulation: Aggregated Estimation}


In this simulation experiment, we look into the performance of the aggregated estimation introduced in Section~\ref{sec:averaging}. The signal part $\bm X\in\mathbb R^{2^M\times 2^N}$ is generated as a mixture of $k_0$ Kronecker products:
$$\bm X = \sum_{i=1}^{k_0} \bm A_i\otimes \bm B_i,$$
where $\bm A_i$ and $\bm B_i$ are $2^{m_i}\times 2^{n_i}$ and $2^{M-m_i}\times 2^{N-n_i}$ matrices, correspondingly. In other words, the $i$-th Kronecker product in $\bm X$ has the configuration $(m_i, n_i)$. (Here for notational simplicity, we use $(m,n)$ instead of $(2^m, 2^n)$ to indicate the configuration.) We design the structure of $\bm A_i$ and $\bm B_i$ as follows:
\begin{align*}
    \bm A_i &= \varphi \bm D_1^{(i)} \otimes \begin{bmatrix}
    \frac{1}{\sqrt{2}} & \frac{1}{\sqrt{2}}
    \end{bmatrix} + \sqrt{1-\varphi^2} \bm D_2^{(i)} \otimes \begin{bmatrix}
    \frac{1}{\sqrt{2}} & -\frac{1}{\sqrt{2}}
    \end{bmatrix},\\
    \bm B_i &= \varphi \begin{bmatrix}
    \frac{1}{\sqrt{2}} \\ \frac{1}{\sqrt{2}}
    \end{bmatrix}\otimes \bm D_3^{(i)} + \sqrt{1-\varphi^2} \begin{bmatrix}
    \frac{1}{\sqrt{2}} \\ -\frac{1}{\sqrt{2}}
    \end{bmatrix}\otimes \bm D_4^{(i)},
\end{align*}
where $\bm D_1^{(i)}$ and $\bm D_2^{(i)}$ are random $2^{m_i}\times 2^{n_i -1}$ matrices such that they both have Frobenius norm 1 and are orthogonal with each other, and so do $\bm D_3^{(i)}$ and $\bm D_4^{(i)}$.
It follows that $\|\bm A_i\|_F = \|\bm B_i\|_F=1$. The quantity $\varphi^2\in(0, 1)$ controls the representation gap of the product $\bm A_i\otimes \bm B_i$. Specifically, at the configurations $(m_i,n_i-1)$ and $(m_i+1,n_i)$, it holds that
$$\|\mathcal R_{m_i,n_i-1}[\bm A_i\otimes\bm B_i]\|_S\approx \|\mathcal R_{m_i+1,n_i}[\bm A_i\otimes\bm B_i]\|_S\approx \varphi \wedge \sqrt{1-\varphi^2}.$$
Furthermore, under the above construction, the incoherence parameters $\mu$ for $\bm A_i$ and $\bm B_i$ are roughly the same ($\approx \sqrt{(m_i+n_i)\wedge(M+N-m_i-n_i)\cdot \log 2}$) for different values of $\varphi$. 

In this experiment, we set $M=N=9$. Two specific constructions of $\bm X$ are considered: (i) $k_0=1$ (one term model) and $(m_1, n_1)=(5, 4)$, and (ii) $k_0=2$ (two term model) and $(m_1, n_1)=(5, 4)$, $(m_2, n_2)=(4, 5)$. Two values of $\varphi^2=0.5$ and $0.05$ are used, corresponding to large and small representation gaps respectively. We label the four scenarios as follows for later references.\\
\textbf{Scenario L1:} One-term $\bm X$ ($k_0=1$) with a large representation gap $\varphi^2=0.5$.\\
\textbf{Scenario S1:} One-term $\bm X$ ($k_0=1$) with a small representation gap $\varphi^2=0.05$.\\
\textbf{Scenario L2:} Two-term $\bm X$ ($k_0=2$) with a large representation gap $\varphi^2=0.5$.\\
\textbf{Scenario S2:} Two-term $\bm X$ ($k_0=2$) with a small representation gap $\varphi^2=0.05$.\\
For each scenario, the matrix 
$$\bm Y = \bm X + \dfrac{2}{2^{(M+N)/2}}\cdot \bm E$$
is partially observed with observing rate $\tau=0.2$. We apply the aggregation approach \eqref{eq: averaging} proposed in Section~\ref{sec:averaging} to $P_\Omega\bm Y$ with equal weights using up to 10 configurations in the candidate set $\mathcal C_7$. For each configuration, we use either one or two Kronecker products with that configuration in the averaging, and refer to them as K-rank-1 and K-rank-2 in the sequel. 
The mean squared error $2^{-(M+N)}\|\hat{\bm X} - \bm X\|_F^2$ based on 100 repetitions are plotted against the number of configurations in Figure~\ref{fig:error_scenario} for the four scenarios. 
\begin{figure}[!htp]
    \centering
    \includegraphics[width=0.4\textwidth]{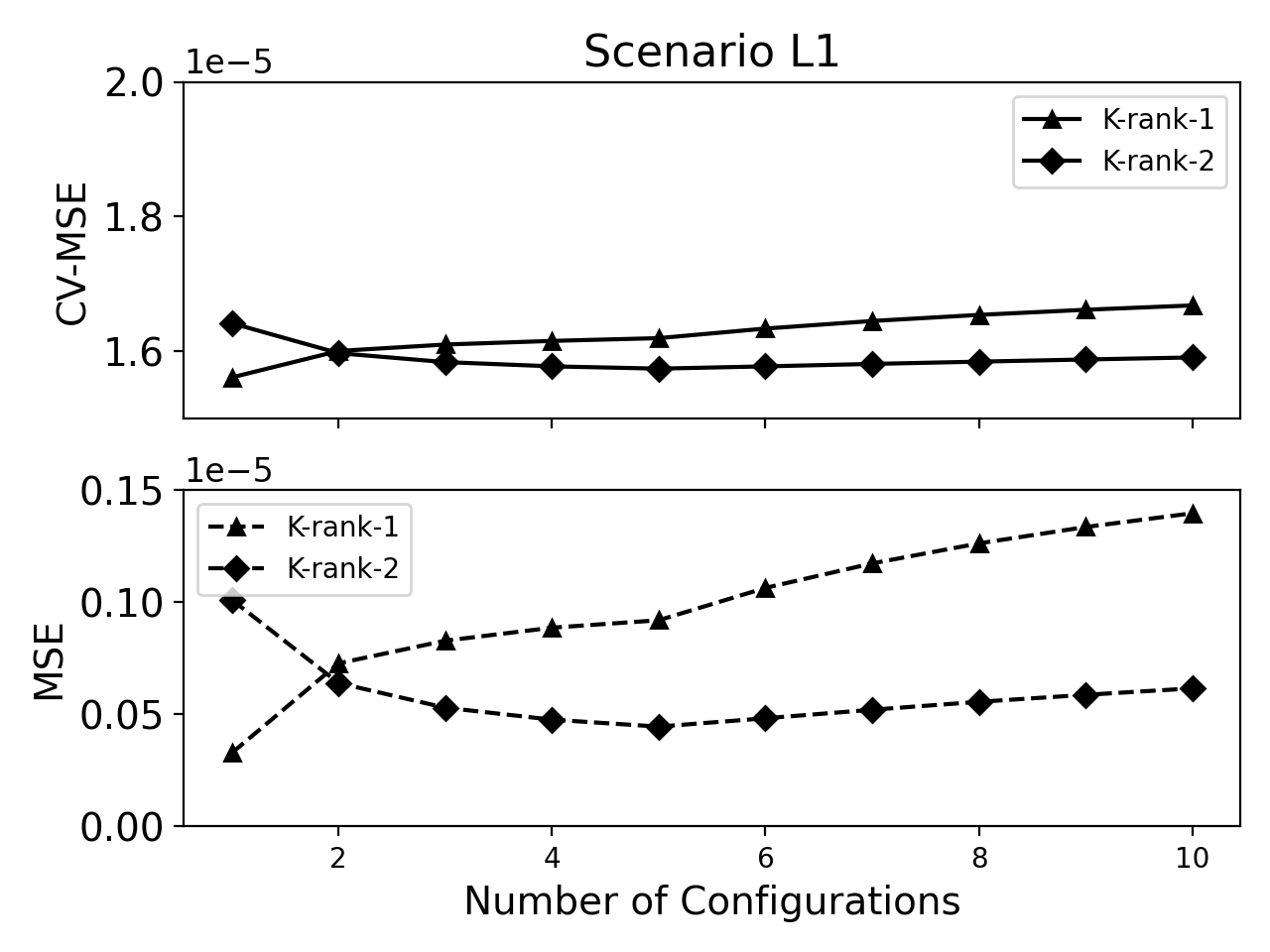}
    \includegraphics[width=0.4\textwidth]{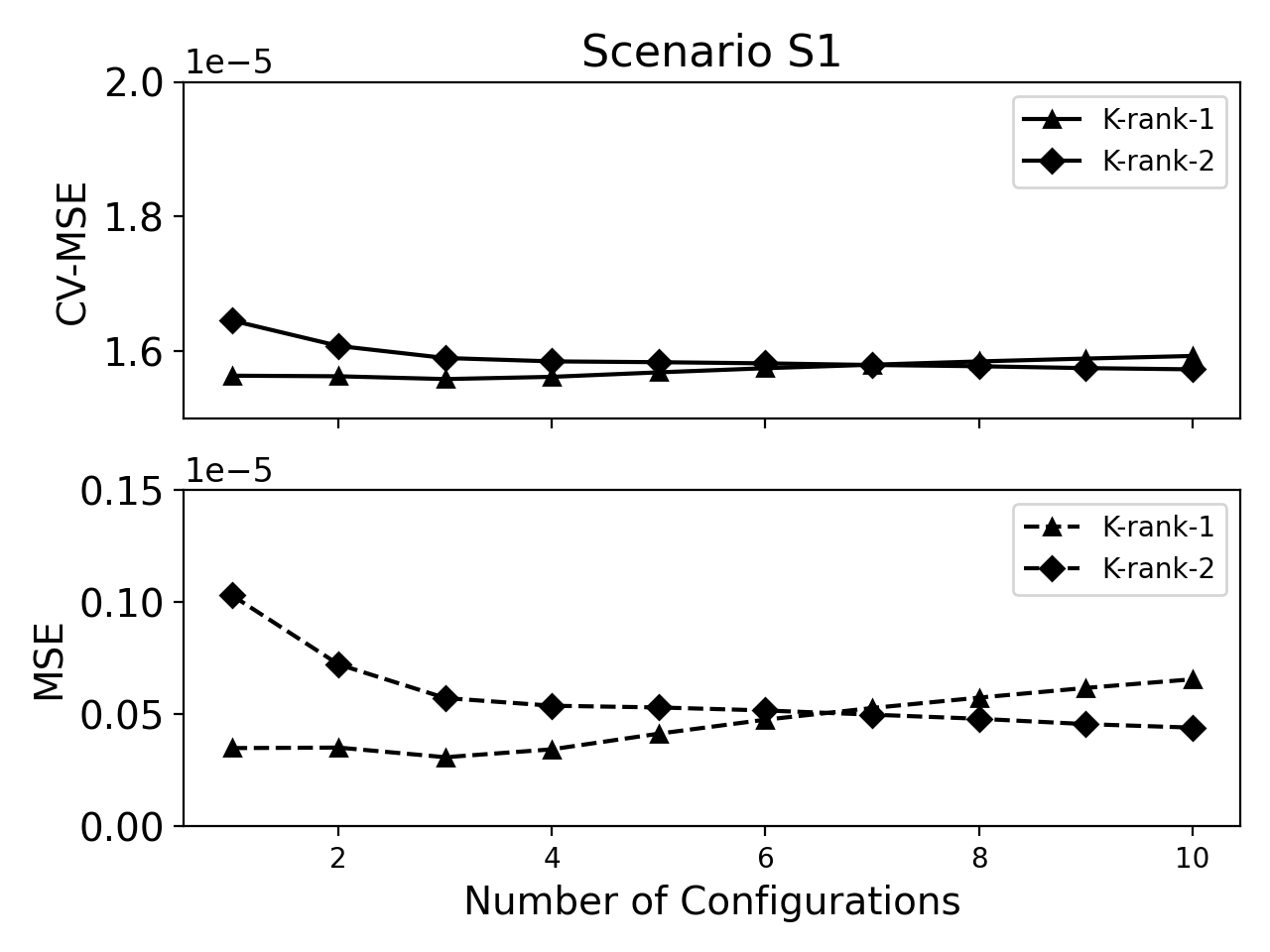}\\
    \includegraphics[width=0.4\textwidth]{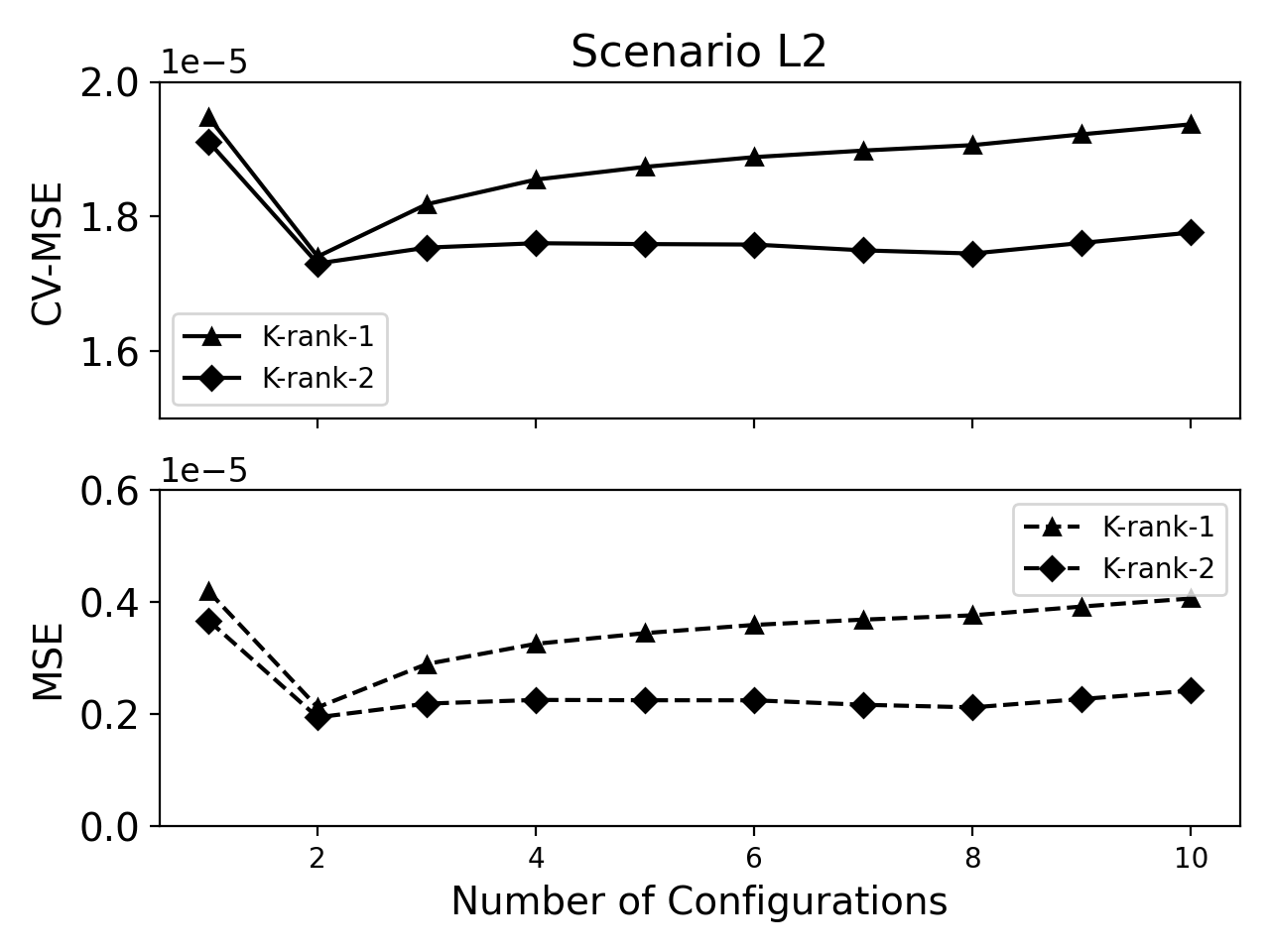}
    \includegraphics[width=0.4\textwidth]{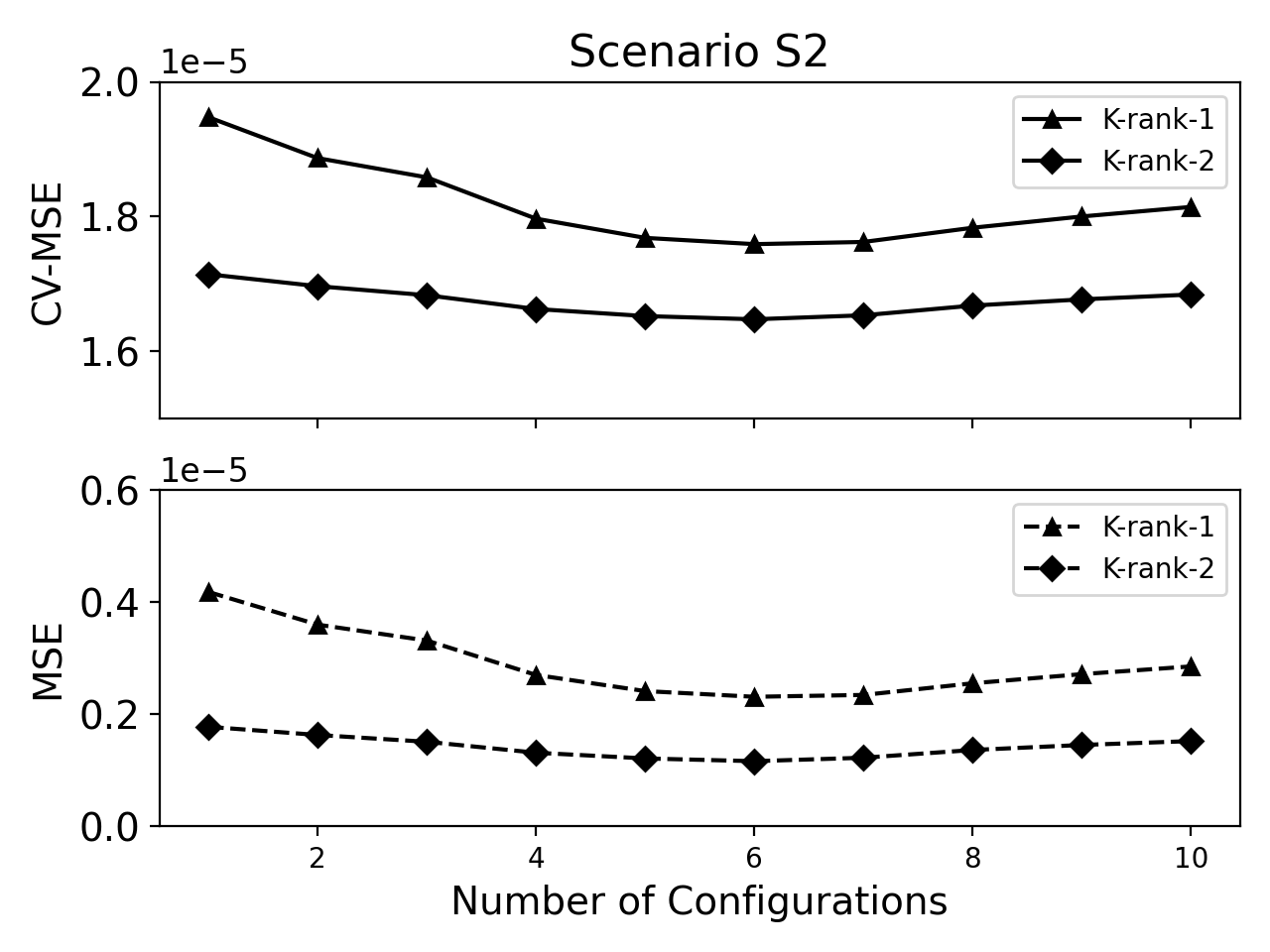}
    \caption{Cross-validated mean squared error (CV-MSE) and mean squared error (MSE) under
four scenarios, using two K-ranks.}
    \label{fig:error_scenario}
\end{figure}

When the representation gap is large, models with wrong configurations usually possess a large bias. The benefit from averaging models under many configurations may not be enough to offset this extra bias. Therefore, in scenarios L1 and L2 (first column of Figure~\ref{fig:error_scenario}), the MSE starts to increase when number of averaged configurations is greater than the true number of Kronecker products $k_0$, except when K-rank-2 models under scenario L1. 
However, when the representation gap is small, there exist several alternative configurations with relatively small bias. Adding those configurations into the averaging model can significantly reduce variance while adding little bias.  Hence, in scenarios S1 and S2 (second column of Figure~\ref{fig:error_scenario}), adding more terms can reduce the MSE. We also observe that in general averaging using K-rank 2 models outperforms averaging K-rank 1 models in S2.

\begin{table}[!tp]
    \centering
    \begin{tabular}{|c|c|c|c|c|c|c|}
         \multicolumn{7}{c}{Scenario L2}\\
         \hline
         \multirow{2}{*}{Term} & \multirow{2}{*}{Config.} & Criterion & \multicolumn{2}{c|}{Error (one term)} & \multicolumn{2}{c|}{Error (agg.)}\\
         \cline{4-7}
         && Function & Rank-1& Rank-2& Rank-1 & Rank-2\\
         \hline
         1 & $(5, 4)$ & 0.2208 & 0.5527 & 0.4939 & 0.5527 & 0.4939\\
         \hline
         2 & $(4, 5)$ & 0.2124 & 0.5547 & 0.4698 & 0.2770 & 0.2520\\
         \hline
         3 & $(5, 5)$ & 0.1773 & 0.7811 & 0.6017 & 0.3792 & 0.2804\\
         \hline
         4 & $(4, 4)$ & 0.1737 & 0.8017 & 0.6226 & 0.4342 & 0.2897\\
         \hline
         5 & $(6, 4)$ & 0.1720 & 0.8294 & 0.6407 & 0.4670 & 0.2934\\
         \hline 
         6 & $(4, 6)$ & 0.1703 & 0.8020 & 0.6470 & 0.4686 & 0.2732 \\
         \hline 
         7 & $(3, 5)$ & 0.1698 & 0.8133 & 0.6491 & 0.4767 & 0.2742\\
         \hline
         8 & $(5, 3)$ & 0.1693 & 0.8303 & 0.6366 & 0.4876 & 0.2693 \\
         \hline
         9 & $(6, 5)$ & 0.1645 & 0.9752 & 0.9893 & 0.5140 & 0.2952 \\
         \hline
         10 & $(5, 6)$ & 0.1606 & 0.9575 & 0.9429 & 0.5327 & 0.3172\\
         \hline
    \end{tabular}
    \caption{
    Configuration-wise MSE and the MSE of averaging over first $k$ configurations are listed for both K-rank 1 and K-rank 2 matrix completion. 
    Error (one term) refers to MSE of $\hat{\bm X}$ obtained in the $k$-th iteration. Error (agg.) refers to the MSE of the aggregated model using the first $k$ configurations. Both are normalized by the factor $\|\bm X\|_F^2/2^{M+N}$.}
    \label{tab:L2}
\end{table}

The ordered configurations, their corresponding MSE, and the MSE of the model aggregated up to it are reported in Table~\ref{tab:L2} for Scenario L2. Under L2, $\bm X$ is the sum of two Kronecker products of configurations $(4, 5)$ and $(5, 4)$, which correspond to the two largest values of the selection criterion, and hence take up the first two rows of Table~\ref{tab:L2}. The gap between the second configuration $(4, 5)$ and the third $(5, 5)$ in terms of the criterion is much greater than the gaps between any other two consecutively selected configurations. As a result, aggregation over only the first two configurations yields the best performance for both K-rank 1 and K-rank 2 matrix completion. 

\begin{table}[!tp]
    \centering
    \begin{tabular}{|c|c|c|c|c|c|c|}
         \multicolumn{7}{c}{Scenario S1}\\
         \hline
         \multirow{2}{*}{Term} & \multirow{2}{*}{Config.} & Criterion & \multicolumn{2}{c|}{Error (one term)} & \multicolumn{2}{c|}{Error (agg.)}\\
         \cline{4-7}
         && Function & Rank-1& Rank-2& Rank-1 & Rank-2\\
         \hline
         1 & $(5, 4)$ & 0.2177 & 0.0868 & 0.2599 & 0.0868 & 0.2599\\
         \hline
         2 & $(6, 4)$ & 0.2163 & 0.1662 & 0.3374 & 0.0931 & 0.1818\\
         \hline 
         3 & $(5, 3)$ & 0.2143 & 0.1597 & 0.2958 & 0.0792 & 0.1418\\
         \hline 
         4 & $(6, 3)$ & 0.2077 & 0.1926 & 0.3200 & 0.0888 & 0.1364\\
         \hline 
         5 & $(6, 5)$ & 0.1799 & 0.7134 & 0.4740 & 0.1088 & 0.1391\\
         \hline 
         6 & $(7, 4)$ & 0.1799 & 0.6758 & 0.4577 & 0.1247 & 0.1414\\
         \hline 
         7 & $(4, 3)$ & 0.1777 & 0.6736 & 0.4636 & 0.1342 & 0.1296\\
         \hline
         8 & $(5, 2)$ & 0.1731 & 0.7377 & 0.4502 & 0.1466 & 0.1234\\
         \hline 
         9 & $(5, 5)$ & 0.1712 & 0.5806 & 0.2337 & 0.1593 & 0.1164\\
         \hline 
         10 & $(4, 4)$ & 0.1711 & 0.5623 & 0.2405 & 0.1666 & 0.1113\\
         \hline
    \end{tabular}
    \caption{Top 10 configurations according to the criterion function under Scenario S1. Error (one term) refers to MSE of $\hat{\bm X}$ obtained in the $k$-th iteration. Error (agg.) refers to the MSE of the aggregated model using the first $k$ configurations. Both are normalized by the factor $\|\bm X\|_F^2/2^{M+N}$. }
    \label{tab:S1}
\end{table}

Similar results for Scenario S1 are reported in Table~\ref{tab:S1}. Due to the small representation gap, configurations $(6, 4)$, $(5, 3)$ and $(6, 3)$ are close to the truth $(5, 4)$ in terms of the the criterion function. Therefore, aggregation over more terms can reduce the MSE. 

The number of terms in the averaging can be determined by a cross-validation procedure as discussed in Section~\ref{sec:averaging}. We conduct a 10-fold cross-validation and report the cross-validated MSE (CV-MSE) based on 100 repetitions in Figure~\ref{fig:error_scenario}. The CV-MSE exhibits patterns similar to MSE and can approximate MSE well (up to an offset) in Figure~\ref{fig:error_scenario}, suggesting that cross-validation can be used in practice to determine the number of terms incorporated in the aggregated recovery. 
For Scenarios S1, L1 and L2, for all 100 repetitions, the number of configurations that minimizes the CV-MSE also minimizes the MSE, while for S2, 92 out of 100 repetitions do. 


As we mentioned earlier, aggregation is an approach that tunes the bias-variance trade-off, which results in a smaller error when (i) the added terms can reduce the bias or (ii) the added terms can reduce variance while only introducing additional bias slightly. In the preceding experiments, if $\bm X$ is a sum of multiple Kronecker products of different configurations, averaging over multiple configurations can reduce the bias but the bias begins to increase when number of averaged models exceeds the number of terms in $\bm X$ as in Scenario L2. On the other hand, if the terms in $\bm X$ have small representation gaps, aggregation can improve the overall performance by reducing the variance without increase the bias significantly as in Scenarios S1 and S2.

\subsection{Simulation: Completely missing block}
In addition to variance reduction, the aggregation over different configurations can also help recovering entries that are unrecoverable under a single configuration. 
In this part of the simulation study, we consider Scenario S1 in Section~\ref{sec:averaging}, and force that the first $2^4\times 2^5$ block is completely unobserved. 
In other words, the index set of observed entries reduces from the original $\Omega$ to $\Omega^B = \Omega\setminus B$, where $B:=\{(i,j): i=1,\cdots, 2^4;j=1,\cdots, 2^5\}$.
We repeat the whole simulation 100 times and record the mean squared error of all the entries of $\bm X$, and of the first block (which is completely missing) of $\bm X$. Specifically, the overall MSE is $2^{-(M+N)}\|\hat{\bm X}-\bm X\|_F^2$ and the MSE of the first block is $2^{-(M+N-m_1-n_1)}\|P_B[\hat{\bm X}] - P_B[\bm X]\|_F^2$, where $P_B$ is the projection operator onto the set $B$. 
The benchmark method fills the missing entries by the average value of observed ones. If there are still irrecoverable entries with K configurations, by default the benchmark method is used to fill in the values.
We report the median of the MSE over 100 repetitions as a function of the number of configurations in Figure~\ref{fig:missing_block}. Median instead of average is reported here for the sake of robustness, since in several repetitions the MSE is tremendously large for the first block. 
Although the first block is irrecoverable under the true configuration $(5, 4)$ (and also irrecoverable under $(m, n)$ with $m \geqslant 5$ and $n\geqslant 4$), Figure~\ref{fig:missing_block} reveals that it can be recovered partially by the aggregated estimation, where the MSE of the first block is typically greater than the overall MSE. 

\begin{figure}[!htb]
    \centering
    \includegraphics[width=0.45\textwidth]{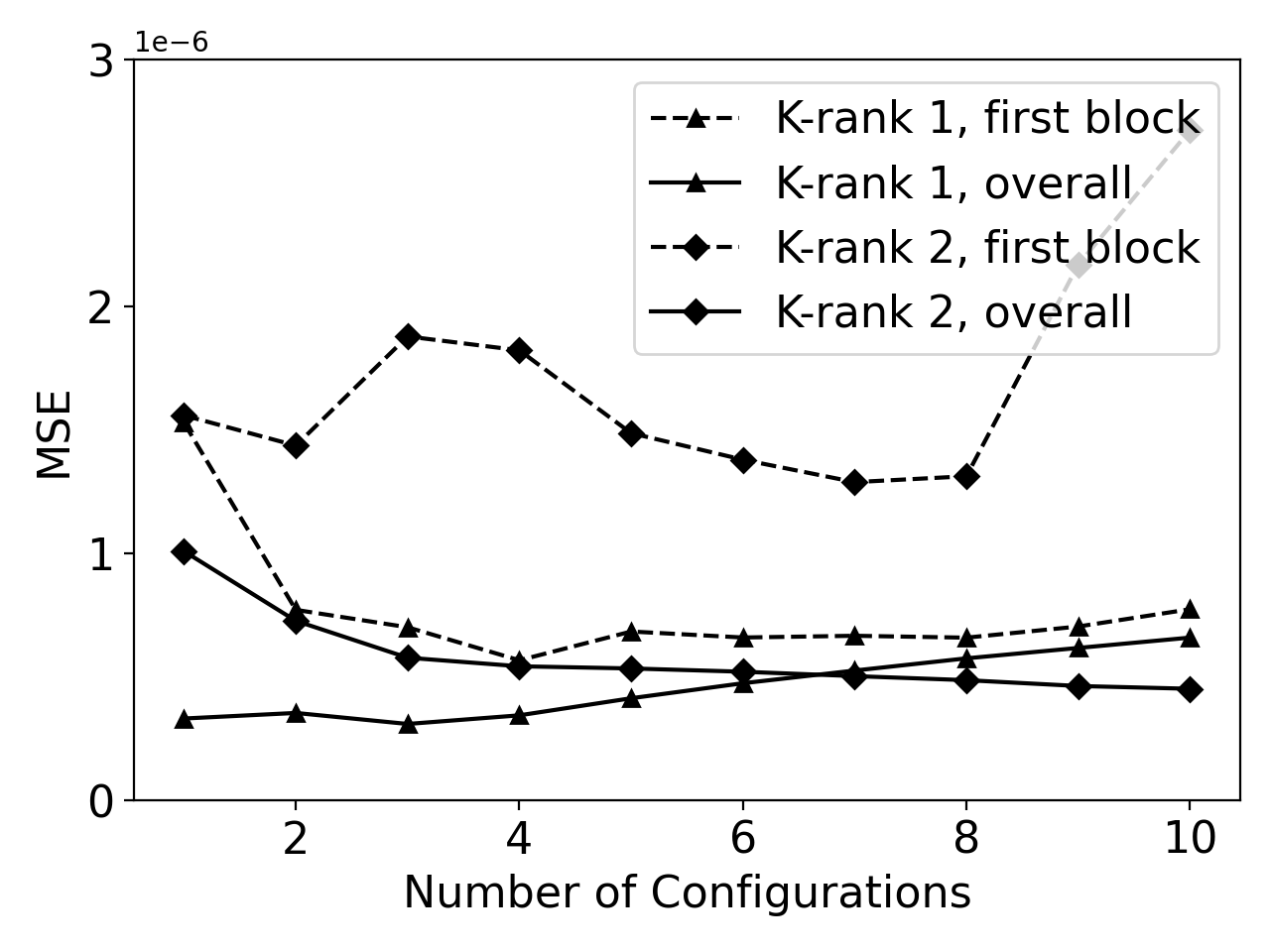}
    \caption{Median of overall MSE and first-block MSE over 100 repetitions as a function of number of configurations for Scenario S1.}
    \label{fig:missing_block}
\end{figure}

\begin{table}[!htb]
    \centering
    \begin{tabular}{|c|c|c|c|c|c|c|}
    \hline
    \multirow{2}{*}{Term} & \multirow{2}{*}{Config.} & Criterion & \multicolumn{2}{c|}{MSE (overall)} & \multicolumn{2}{c|}{MSE (1st block)}\\
    \cline{4-7}
    & & Function & Rank-1 & Rank-2 & Rank-1 & Rank-2\\
    \hline
    1 &    (5, 4) & 0.2167 & 0.0914 & 0.2645 & 0.9986 & 0.9986 \\
    \hline
    2 &    (6, 4) & 0.2153 & 0.0977 & 0.1868 & 0.9986 & 0.9986 \\
    \hline
    3 &    (5, 3) & 0.2131 & 0.0795  & 0.1423 & 0.1191  & 0.1438 \\
    \hline
    4 &    (6, 3) & 0.2066 & 0.0889  & 0.1372 & 0.0847  & 0.1403 \\
    \hline
    5 &    (7, 4) & 0.1793 & 0.1062  & 0.1391 & 0.0847  & 0.1403 \\
    \hline
    6 &    (6, 5) & 0.1793 & 0.1248  & 0.1420 & 0.0847  & 0.1403 \\
    \hline
    7 &    (4, 3) & 0.1764 & 0.1344  & 0.1302 & 0.0710  & 0.1200 \\
    \hline
    8 &    (5, 2) & 0.1721 & 0.1472  & 0.1239 & 0.1204  & 0.1234 \\
    \hline
    9 &    (5, 5) & 0.1706 & 0.1598  & 0.1168 & 0.1204  & 0.1234 \\
    \hline
    10 &    (4, 4) & 0.1702 & 0.1685  & 0.1115 & 0.4567  & 0.0900 \\
    \hline
    \multicolumn{3}{|c|}{Benchmark (Fill in mean)} & \multicolumn{2}{c|}{1.599} & \multicolumn{2}{c|}{0.9986}\\
    \hline
    \end{tabular}
    \caption{Top 10 configurations according to the criterion function under Scenario S1 with a completely missing block under true configuration. the overall MSE and the one for the first block are normalized by $\|\bm X\|_F^2/2^{M+N}$ and $\|P_B\bm X\|_F^2/2^{(M+N-m_1-n_1)}$ accordingly.}
    \label{tab:missing_block}
\end{table}

We pick one repetition and report the top 10 configurations according to the criterion function, and the corresponding MSE up to the given configuration in Table~\ref{tab:missing_block}. 
The performance of the benchmark method is also reported. 
The top 10 configurations are the same as those in Table~\ref{tab:S1}, but the values of the criterion function are slightly smaller due to the extra missing block. The first block is irrecoverable under the top 2 configurations and therefore a benchmark method is used to fill in the first block, resulting in a large MSE for the first two configurations.
Starting from the third row, when aggregating over more than three configurations, the first block is recoverable, with a significantly smaller MSE. 


\subsection{Real image example}
In this section, we apply the proposed matrix completion approach to the cameraman's image, which has been widely used as a benchmark in image analysis.
The original image, which is a $512\times 512$ gray-scaled picture, is shown in Figure \ref{fig:cameraman} (left panel).
It is represented by a $512\times 512$ matrix $\bm X$ of real numbers between 0 and 1, 0 for black and 1 for white.

We first add a noise to the original image such that
$\bm Y = \bm X +0.1 \bm E,$
where the entries of $\bm E$ are IID standard Gaussian noises.
The corrupted image $\bm Y$ is shown in the middle panel
of Figure \ref{fig:cameraman}. We set the observing rate to 20\% ($\tau=0.2$) and
sample the observing set $\Omega$ as i.i.d Bernoulli($\tau$).
The observed $\bm Y^*=P_\Omega \bm Y$ is plotted in the
right panel of Figure \ref{fig:cameraman}, where missing entries are filled as white.
\begin{figure}[tb]
    \vskip 0.2in
    \centering
    \includegraphics[width=0.32\columnwidth]{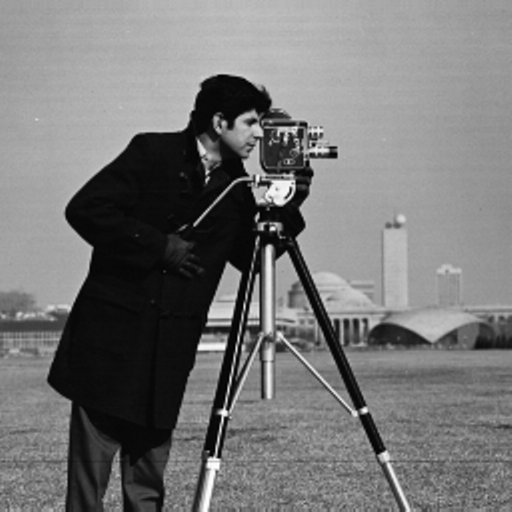}
    \includegraphics[width=0.32\columnwidth]{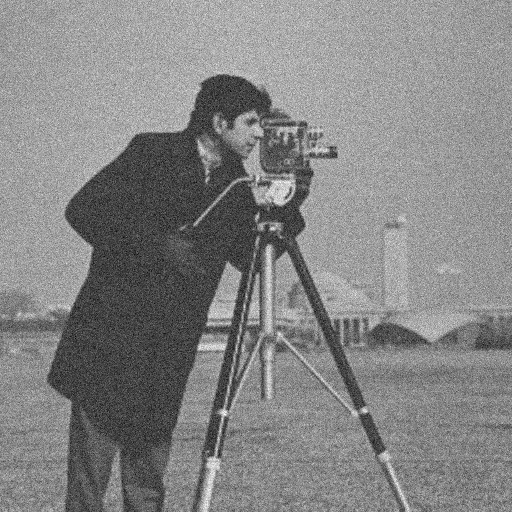}
    \includegraphics[width=0.32\columnwidth]{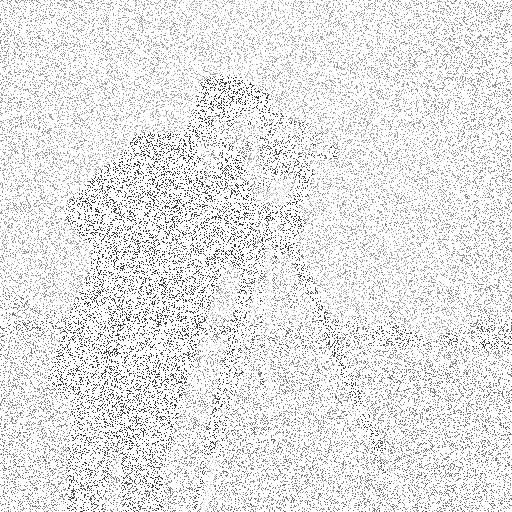}
    \caption{(Left) Cameraman's image; (Middle) Cameraman's image with noise; (Right) Noisy image with 20\% observed entries.}
    \label{fig:cameraman}
\end{figure}
We follow the configuration determination procedure proposed in Section \ref{sec:method-config}. 
The maximum of the criteria function $\|\mathcal R_{p,q}[P_\Omega \bm Y]\|_S$ is attained at the configuration $(\hat p, \hat q) = (64, 64)$, within the candidate configuration set $\mathcal C_6$ (defined in Section~\ref{sec:simulation}). It corresponds to decomposing $\bm X$ as the Kronecker product of a $64\times 64$ and a $8\times 8$ matrix.
Recovered images under the configuration (64,64) and K-ranks 1 to 3 are shown
in the upper row of Figure \ref{fig:recovered}. The cameraman can be recognized from
the recovered matrix of K-rank 1 and more details are added as
the K-rank increases to 2 and 3.

We also compare the performance of matrix completion through KPD with the classical approach through SVD. 
In particular, we consider matrix completion via the alternating minimization algorithm with ranks 8, 16 and 24, matching the numbers of parameters under K-ranks 1, 2, 3 of the Kronecker matrix completion with configuration $(64, 64)$. These recovered images are shown in the lower row of Figure \ref{fig:recovered}. 
The superiority of the KPD approach is easily seen from the images.
Besides judging the recovered images by eyesight,
we quantify the quality of a recovered matrix by the {\it reconstruction error}
\begin{equation}\label{eq:recon_error}
    \|\bm X-\hat{\bm X}\|_F^2/\|\bm X\|_F^2,
\end{equation}
where $\bm X$ is the image without noise and $\hat{\bm X}$ is the recovered matrix. 
Table \ref{tab:lenna-comparison} reports the reconstruction errors of recovered matrices through KPD and SVD. The trimmed error is for the trimmed recovered matrix, whose pixel values are restricted to $[0, 1]$. It confirms again that Kronecker matrix completion can recover the cameraman's image more accurately compared to the SVD approach, with a similar number of parameters. The result is anticipated since KPD matrix completion has more flexibility in selecting the configurations, including SVD matrix completion as one of its special cases. The proposed configuration determination procedure is able to find a better configuration, which provides better performance than
one of its special cases. 
\begin{table}[!tb]
    \centering
    \begin{tabular}{|c|c|c|c|}
    \hline
    KPD rank & 1 & 2 & 3  \\
    \hline
    Error & 0.1224 & 0.1084 & 0.1136\\
    Error (trimmed) & 0.1219 & 0.1060 & 0.1071\\
    \hline
    SVD rank & 8 & 16 & 24\\
    \hline
    Error & 0.2167 & 0.4986 & 0.9968\\
    Error (trimmed) & 0.2060 & 0.3464 & 0.5988\\
    \hline
    \end{tabular}
    \caption{Error for Kronecker matrix completion and classical matrix completion with similar number of parameters.}
    \label{tab:lenna-comparison}
\end{table}
\begin{figure}[!tb]
    \centering
    \includegraphics[width=0.3\columnwidth]{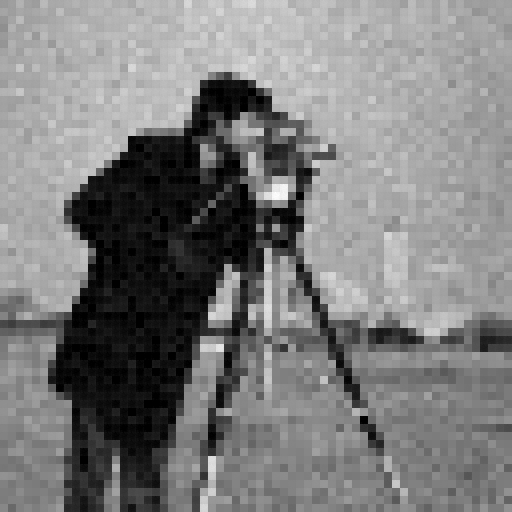}
    \includegraphics[width=0.3\columnwidth]{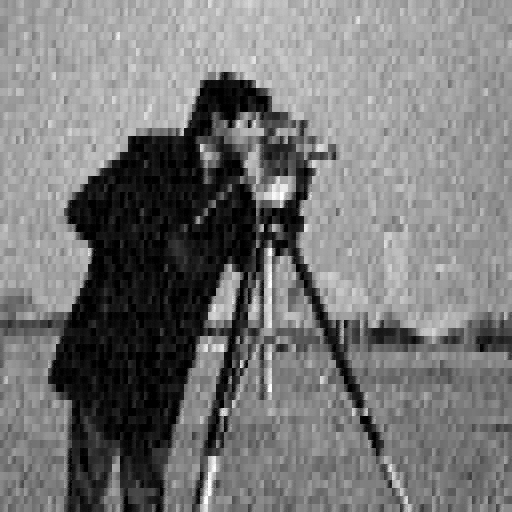}
    \includegraphics[width=0.3\columnwidth]{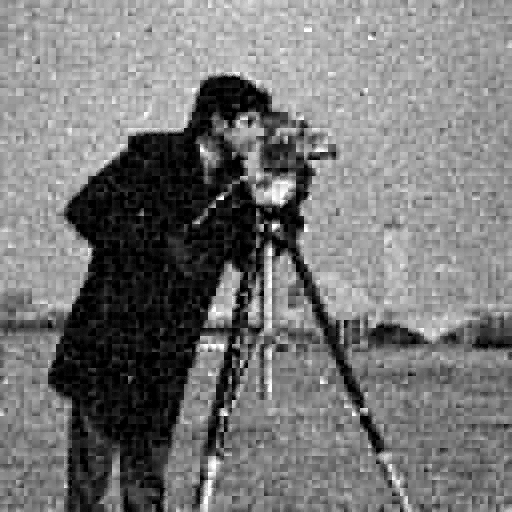}\\
    \includegraphics[width=0.3\columnwidth]{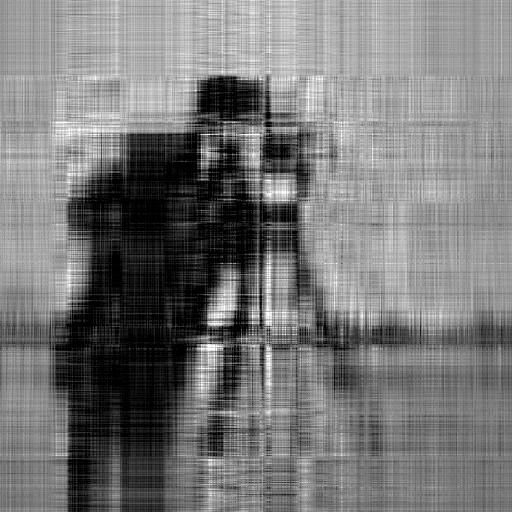}
    \includegraphics[width=0.3\columnwidth]{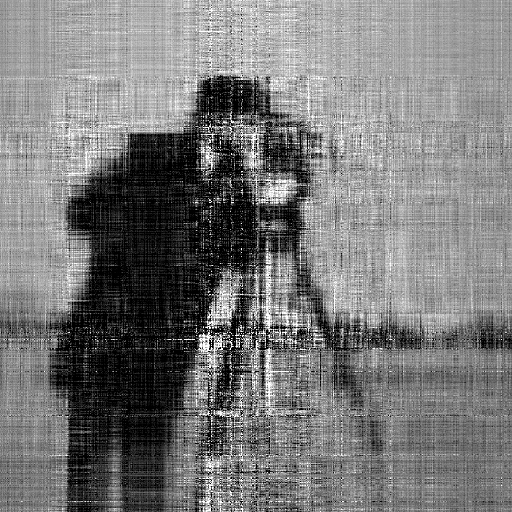}
    \includegraphics[width=0.3\columnwidth]{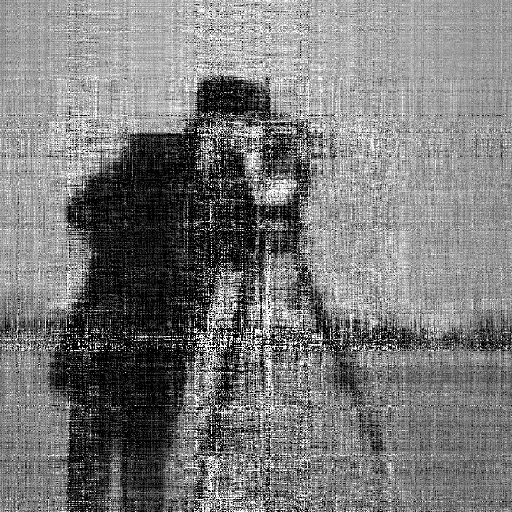}\\
    \caption{(Upper row) Recovered images using KPD (Lower row) Recovered images using SVD. The extreme values are trimmed such that all pixel values are between 0 and 1.}
    \label{fig:recovered}
\end{figure}

To examine the performance of the aggregation approach, we consider three restricted candidate sets 
$\mathcal C_5, \mathcal C_6, \mathcal C_7$, defined in Section~\ref{sec:simulation}, corresponding to $pq\wedge p^*q^*\geqslant 32, 64, 128$, respectively. 
We also consider both K-ranks 1 and 2 for each configuration. The configurations within the candidate sets are ordered based on the criterion function. 
Then $\hat{\bm X}$ is obtained by a simple average of the matrices recovered under the first $d$ configurations, according to \eqref{eq: averaging}. The left sub-figure of Figure~\ref{fig: error_average} shows the reconstruction error against $d$. The errors corresponding to K-rank 2 for the candidate set $\mathcal C_5$ is much worse than all other scenarios and is therefore not shown. 

The right sub-figure of Figure~\ref{fig: error_average} reports the cross-validation errors. We choose to use $K=20$-fold to minimize the impact of decreased observing rate. Although CV-MSE and the reconstruction error have different scales (the former is calculated according to \eqref{eq:cv-mse} where the error $\bm E$ is also involved, the latter does not involve $\bm E$ directly and is normalized by $\|\bm X\|_F^2$ as in \eqref{eq:recon_error}), CV-MSE nevertheless exhibits the same trend as the reconstruction error and the number of configurations that minimize CV-MSE coincides with the one that minimize the reconstruction error for each scenario. Similar to the simulation studies, the CV-MSE can work as a proxy of the reconstruction error, which is usually infeasible in practice. In addition, CV-MSE helps not only in determining the number of configurations with respect to 
a certain configuration set $\mathcal C$, but also in choosing the appropriate fitting K-rank and choosing the best configuration set $\mathcal C$. 

We note that 
determining the best configuration set is not trivial. On the one hand, as discussed in Section~\ref{sec:analysis}, extreme configurations ($pq < (PQ)^{1/4}$ or $p^*q^*<(PQ)^{1/4}$) must be excluded to have a stable Kronecker product matrix completion. On the other hand, configurations close to these boundaries ($pq = (PQ)^{1/4}$ or $p^*q^*=(PQ)^{1/4}$) may have uncontrollable performance 
and slower convergence rates. 
Fitting a higher K-rank matrix completion model with those configurations can potentially suffer from sever overfitting. 
For example, if only K-rank 1 models are considered, $\mathcal C_5$ performs the best. However, fitting K-rank 2 models with respect to $\mathcal C_5$ results in a tremendously large error in some unobserved entries and is therefore not reported in Figure~\ref{fig: error_average}.


It can be seen from Figure~\ref{fig: error_average} that using K-rank 2 and candidate set $\mathcal C_6$ performs the best. ($\mathcal C_5$ with K-rank 2 is severly overfitted as discussed earlier) 
The error of averaging the top 5 configurations is around 0.0790, which is about 20\% smaller than 0.1084, obtained by the K-rank two matrix completion under a single configuration, as reported in Table~\ref{tab:lenna-comparison}. The error rate then stays roughly constant if more configurations are added into the averaging. 



The reconstructed images averaged over 4 configurations with K-rank 1 under the candidate sets 
$\mathcal C_5$ and over 5 configurations of K-rank 2 under $\mathcal C_6$  are shown in Figure~\ref{fig:averge_image}. With K-rank 1 (left panel), big pixels are observed in the reconstructed image but are less noisy than the ones in Figure~\ref{fig:recovered}. With K-rank 2 (right panel), more details are added, resulting in a smoother reconstructed image.
\begin{figure}[!htb]
    \centering
    \includegraphics[width=0.4\columnwidth]{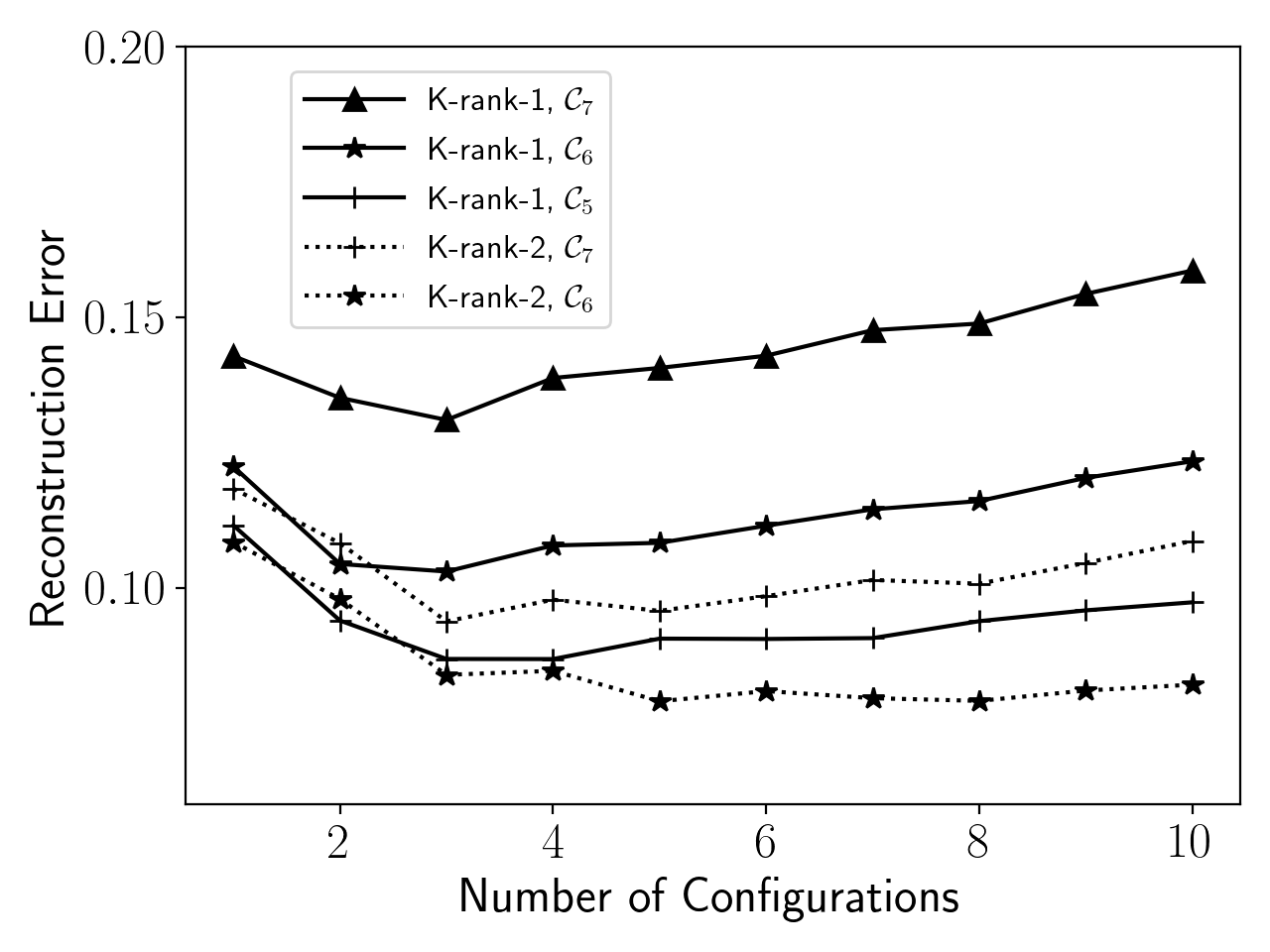}
    \includegraphics[width=0.4\columnwidth]{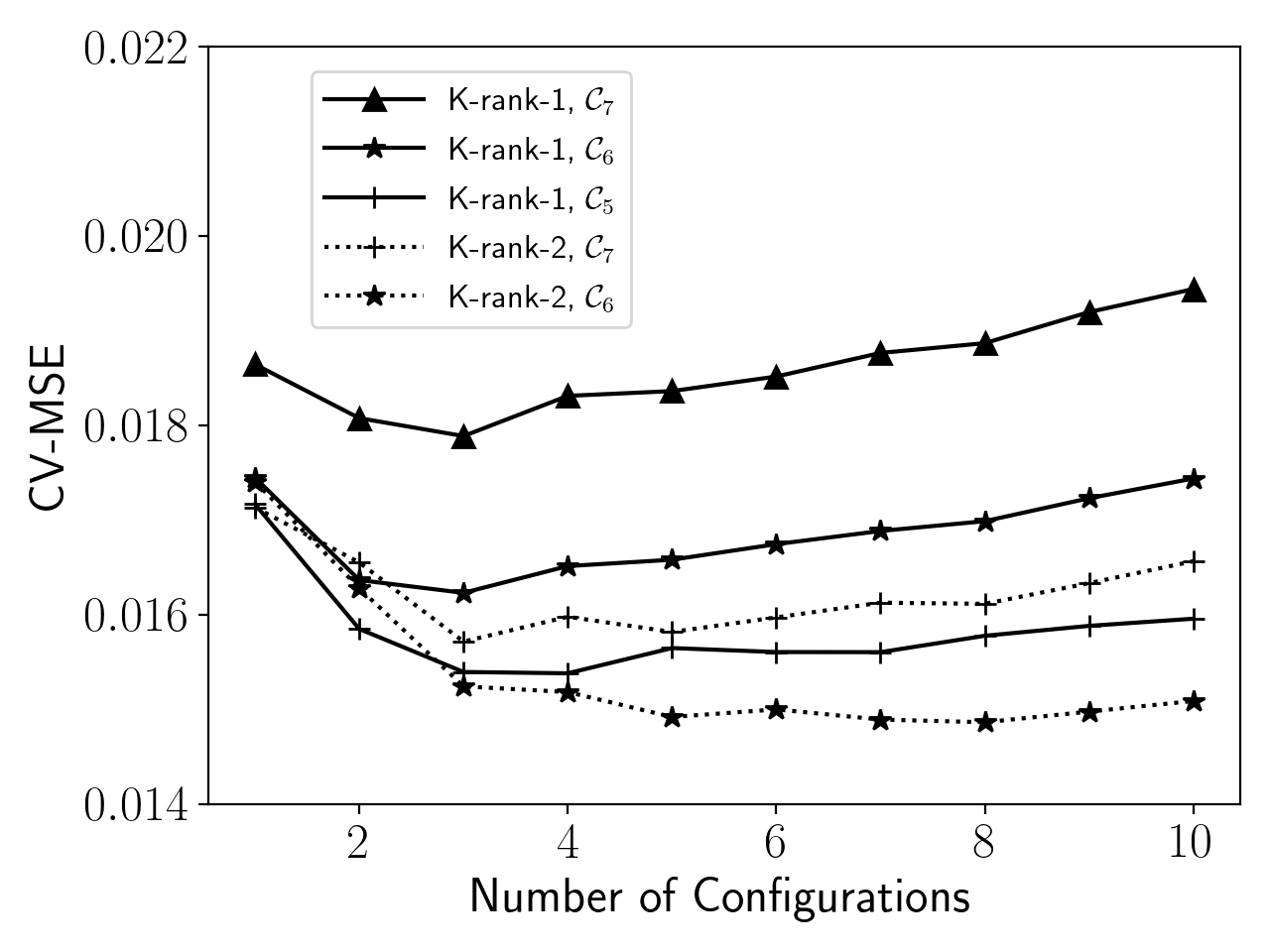}
    \caption{(Left) Reconstruction error and (Right) cross-validation error (CV-MSE) of the averaged matrix against the number of configurations.}
    \label{fig: error_average}
\end{figure}
\begin{figure}[!htb]
    \centering
    \includegraphics[width=0.4\columnwidth]{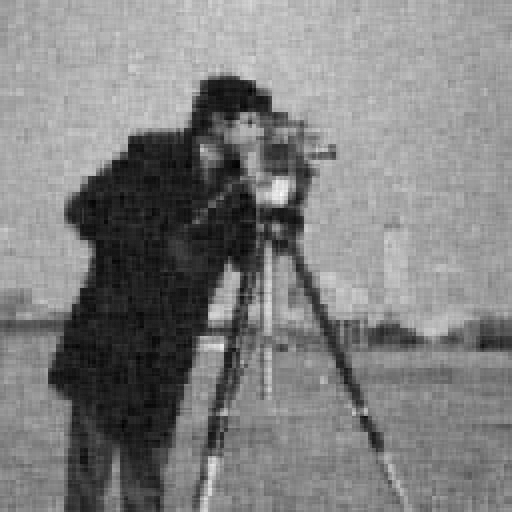}\hskip 0.05\columnwidth
    \includegraphics[width=0.4\columnwidth]{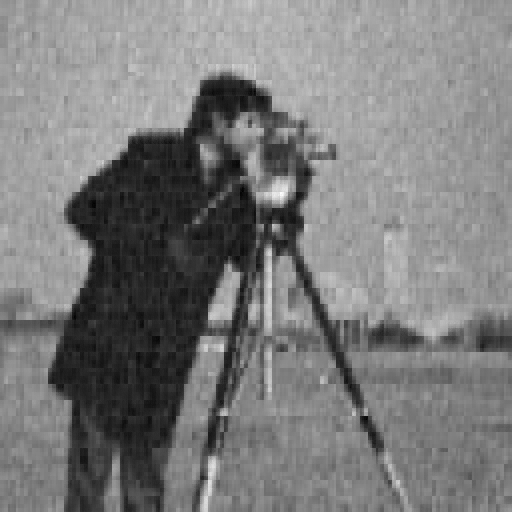}
    \caption{(Left) Average reconstructed image over 4 configurations of K-rank 1. (Right) Average reconstructed image over 5 configurations of K-rank 2. }
    \label{fig:averge_image}
\end{figure}

\section{Conclusion}\label{sec:conclusion}
In this article, we study matrix completion problem assuming the underlying complete signal matrix is of a low 
rank Kronecker product form, extending the classical assumption of low rank signal matrix. The new model includes the low rank model
as a special case. Such an extension brings a significantly greater modeling flexibility which leads to more effective dimension reduction, a flexible mechanism for aggregated recovery,
and a significant reduction in the number of unrecoverable entries when the observing rate is very low.
It also allows a wide range of applications in signal processing, image analysis and many other fields, where data often appear to have
the Kronecker product structure, 
which the classical low rank matrix completion methods may not be able to handle effectively.
We also propose a MSE based criterion to determine the unknown configuration of the underlying Kronecker product. 

There are a number of directions to further extend and explore the capability of the Kronecker product based matrix completions. First, it is of interest to devise a procedure/algorithm for the joint selection of the configuration and the K-rank. 
Second, to demonstrate the potential advantage of the aggregated estimation, we have used the equal weight of different configurations. A better aggregation can possibly be achieved by using weighted average, where the weights reflect the accuracy of different models, and can be determined by the cross validation. Last but not least, Similar structure can be introduced for tensors, leading to new tensor completion methods. We anticipate that the algorithm and the analysis will be substantially different, because it is known that tensor decomposition is very different from the matrix decomposition. We plan to follow out these ideas in the future work.


\bibliography{reference}
\bibliographystyle{apalike}

\newpage
\appendix 

\centerline{\bf \Large Appendix}

\bigskip

We make the convention that $C, C_1, C_2, \ldots$ denote absolute constants, whose values may change from place to place.

\section{Proof of Theorem~\ref{thm:gap}}
We first prove several technical lemmas. 

\begin{lem}[Over-representation]\label{lem:over-representation}
Let $\bm Q$ be a $m\times n$ matrix with IID Bernoulli random variables with success rate $\tau$. Define
$$\mathcal A = \left\{i\in[m]:\sum_{j=1}^n Q_{ij} < 2\tau n\right\}.$$
Then we have \
$$P[|\mathcal A|=m]\geqslant 1- \exp\left\{\log m - \dfrac{3n\tau}{8(1-\tau)}\right\}.$$
\end{lem}
\begin{proof}
If $\tau > 1/2$, then $|\mathcal A|=m$ has probability one. Therefore, we only consider $\tau \leqslant 1/2$. Let $Z_{ij} = Q_{ij}-\tau$. 
We have $\mathbb E[Z_{ij}] = 0$, $\mathrm{Var}[Z_{ij}] = \tau(1-\tau)$ and $|Z_{ij}| \leqslant 1-\tau$. By Bernstein's inequality, we have
$$P\left[\sum_{j=1}^n{Z_{ij}} \geqslant t\right]\leqslant \exp\left\{-\dfrac{t^2/2}{n\tau(1-\tau) + (1-\tau)t/3}\right\}.$$
Therefore
$$P\left[\sum_{j=1}^n{Q_{ij}} \geqslant 2\tau n\right]\leqslant \exp\left\{-\dfrac{3n\tau}{8(1-\tau)}\right\}.$$
Using union bound, we have
$$P\left[\max_{i\in[m]}\sum_{j=1}^n{Q_{ij}} \geqslant 2\tau n\right]\leqslant \exp\left\{\ln m-\dfrac{3n\tau}{8(1-\tau)}\right\}.$$
\end{proof}
The matrix $\bm Q$ is \textit{row over-represented} if $|\mathcal A|< m$. Similarly, a matrix $P_\Omega\bm M$ is over-represented if either $P_\Omega\bm M$ is row over-represented or $[P_\Omega\bm M]^T$ is row over-represented. We restate the theorem on $\|P_\Omega \bm M\|_S$ from \citet{keshavan2010matrixcompletion} in Lemma~\ref{lem:restate-bound}. Recall that $\|\cdot\|_{\max}$ denotes the maximum absolute entry of a matrix.

\begin{lem}[\citet{keshavan2010matrixcompletion}]\label{lem:restate-bound}
Suppose $\bm M$ is a $m\times n$ ($m\geqslant n$) matrix with elements observed IID with probability $\tau$. Denote the partially observed matrix as $P_\Omega\bm M$. If $P_\Omega\bm M$ is not over-represented, then there exists a constant $C>0$ such that, with probability larger than $1-1/n^3$
$$\dfrac{\left |\|P_\Omega \bm M\|_S - \tau \|\bm M\|_S\right|}{\sqrt{\tau mn}\|\bm M\|_{\max}}\leqslant Cm^{1/4}n^{-3/4}.$$
\end{lem}

\begin{lem}[Max bound on Gaussian matrix]\label{lem:gaussian-max}
If $\bm E$ is a $m\times n$ matrix with IID standard Gaussian random variables, then there exists a constant $C >0 $ such that, with probability larger than $1-1/(mn)^3$,
$$\|\bm E\|_{\max} \leqslant C\sqrt{\log mn}.$$
\end{lem}
\begin{proof}
Since $E_{ij}$ is standard Gaussian, we have $P[|E_{ij}|>t]\leqslant 2\exp\{-t^2/2\}$. Using union bound, we have
$$P[\|\bm E\|_{\max} > t] \leqslant \exp\{\log (2mn) - t^2/2\}. $$
The lemma follows immediately by choosing $C > 2\sqrt{2}$. 
\end{proof}

From Lemma~\ref{lem:over-representation}, we have that for any configuration $(p ,q)\in\mathcal C_\delta$, 
\begin{align*}
P[\mathcal R_{p, q}[P_\Omega\bm Y] \text{ is over-represented}]&\leqslant 2\exp\left\{\left(\dfrac{3}{4}-\delta\right)\log PQ - \dfrac{3\tau}{8(1-\tau)}(PQ)^{1/4+\delta}\right\}\\
&\leqslant \exp\left\{-C_1\tau (PQ)^{1/4+\delta}\right\}
\end{align*}
for some  constant $C_1$. Therefore, from Lemma~\ref{lem:restate-bound}, for $(p, q)\in \mathcal C_\delta$, with probability larger than $1- (PQ)^{-3/4-3\delta} - \exp\{-C_1\tau (PQ)^{1/4+\delta}\}$ we have
\begin{equation}
\left|\|\mathcal R_{p, q}[P_\Omega \bm X]\|_S - \tau \|\mathcal R_{p, q}[\bm X]\|_S\right|\leqslant C_2 \sqrt{\tau}\cdot \sqrt{PQ}\|\bm X\|_{\max}\cdot PQ^{1/4}\cdot (pq\wedge p^*q^*)^{-1},\label{eq:proof-signal-bound-1}
\end{equation}
for some  constant $C_2$, where $\bm X = \lambda \bm A\otimes \bm B$ is the signal part. Noticing that $\|\bm X\|_{\max} = \lambda \|\bm A\|_{\max}\|\bm B\|_{\max}$ and $pq\wedge p^*q^*\geqslant (PQ)^{1/4+\delta}$, \eqref{eq:proof-signal-bound-1} can be further revised to
\begin{equation}
\left|\|\mathcal R_{p, q}[P_\Omega \bm X]\|_S - \tau \|\mathcal R_{p, q}[\bm X]\|_S\right|\leqslant C_2 \lambda\sqrt{\tau}\mu^2\cdot (PQ)^{-\delta}.\label{eq:proof-signal-bound-2}
\end{equation}
Similarly, for the noise part, according to Lemma~\ref{lem:restate-bound} and Lemma~\ref{lem:gaussian-max}, for $(p, q)\in\mathcal C_\delta$ with probability larger than $1 - (PQ)^{-3/4-3\delta} - (PQ)^{-3}$
\begin{equation}
\|\mathcal R_{p, q}[P_\Omega \bm E]\|_S \leqslant \tau \|\mathcal R_{p, q}[\bm E]\|_S + C_3\sqrt{\tau}\sqrt{\log PQ}\cdot (PQ)^{1/2-\delta},\label{eq:proof-noise-bound}
\end{equation}
for some  constant $C_3$. 

For the true configuration $(p_0, q_0)$, using \eqref{eq:proof-signal-bound-2} and \eqref{eq:proof-noise-bound}, we have
\begin{align*}
&\ \mathbb E[\|\mathcal R_{p_0, q_0}[P_\Omega \bm Y]\|_S]\\
&\geqslant \mathbb E[\|\mathcal R_{p_0, q_0}[P_\Omega \bm X]\|_S] - \sigma (PQ)^{-1/2}\mathbb E[\|\mathcal R_{p_0, q_0}[P_\Omega \bm E]\|_S]\\
&\geqslant \left(1- (PQ)^{-3/4-3\delta} - \exp\{-C_1\tau (PQ)^{1/4+\delta}\}\right) \left(\tau\lambda  - C_2\lambda \sqrt{\tau} \mu^2 (PQ)^{-\delta}\right)\\
&\ +  \left((PQ)^{-3/4-3\delta} + \exp\{-C_1\tau (PQ)^{1/4+\delta}\right)\cdot 0 \\
&\ - \sigma (PQ)^{-1/2}\left(1 - (PQ)^{-3/4-3\delta} - (PQ)^{-3}\right)\left(\tau (\sqrt{p_0q_0}+\sqrt{p_0^*q_0^*})+ C_3\sqrt{\tau}\sqrt{\log PQ}(PQ)^{1/2-\delta}\right) \\
&\ - \sigma (PQ)^{-1/2}\left( (PQ)^{-3/4-3\delta} + (PQ)^{-3}\right)\sqrt{\tau PQ}\\
&\geqslant \tau\lambda - O\left(\left(\lambda\mu^2 + \sigma \sqrt{\log PQ}\right)\sqrt{\tau}(PQ)^{-\delta}\right)
\end{align*}
Similarly,  for any $(p, q)\in\mathcal W_\delta$, we have
$$\mathbb E[\|\mathcal R_{p, q}[P_\Omega \bm Y]\|_S]\leqslant \tau\lambda\phi + O\left(\left(\lambda\mu^2 + \sigma \sqrt{\log PQ}\right)\sqrt{\tau}(PQ)^{-\delta}\right).$$
Therefore, 
\begin{multline*}
    \mathbb E[\|\mathcal R_{p_0, q_0}[P_\Omega \bm Y]\|_S] - \max_{(p, q)\in\mathcal W_{\delta}} \mathbb E[\|\mathcal R_{p, q}[P_\Omega \bm Y]\|_S]\\
    \geqslant \tau\lambda(1-\phi)\cdot\left(1 + O\left(\psi^{-2}\left(\mu^2 + \dfrac{\sigma}{\lambda} \sqrt{\log PQ}\right)\tau^{-1/2}(PQ)^{-\delta}\right)\right).
\end{multline*}
Here we use $\psi^2=1-\phi^2\asymp 1-\phi$. Assumption~\ref{assump:snr} ensures the term in big O notation is minor.

\section{Proof of Theorem~\ref{thm:consistency}}
For the true configuration $(p_0, q_0)$, using \eqref{eq:proof-signal-bound-2} and \eqref{eq:proof-noise-bound}, we have
\begin{align*}
&P[G(p_0, q_0) \leqslant \tau\lambda (1 + \phi)/2]\\
\leqslant & P[\|\mathcal R_{p_0, q_0}[\bm E]\|_S \geqslant \sqrt{p_0q_0}+\sqrt{p_0^*q_0^*}+R_1] + (PQ)^{-3/4-3\delta} + \exp\{-C_1\tau (PQ)^{1/4+\delta}\} + (PQ)^{-3}\\
\leqslant & \exp\{-R_1^2/2\}+ 2(PQ)^{-3/4-3\delta} + \exp\{-C_1\tau (PQ)^{1/4+\delta}\} + (PQ)^{-3}\\
\leqslant & C_4(PQ)^{-3/4-3\delta}
\end{align*}
for some  constant $C_4$,
where 
\begin{align*}
R_1&=(PQ)^{1/2}\left[\dfrac{\lambda(1-\phi)}{2\sigma}-(PQ)^{-\delta}\tau^{-1/2}\left(C_2\dfrac{\lambda}{\sigma}\mu^2 + C_3\sqrt{\log PQ}\right)\right]-\sqrt{p_0q_0}-\sqrt{p_0^*q_0^*}\\
&=O\left((PQ)^{1/2}\dfrac{\lambda(1-\phi)}{2\sigma}\right).
\end{align*}
Similarly, for a wrong configuration $(p, q)\in\mathcal W_\delta$, using \eqref{eq:proof-signal-bound-2} and \eqref{eq:proof-noise-bound}, we have
\begin{align*}
&P[G(p, q) \geqslant \tau\lambda (1 + \phi)/2]\\
\leqslant & P[\|\mathcal R_{p, q}[\bm E]\|_S \geqslant \sqrt{pq}+\sqrt{p^*q^*}+R_1] + (PQ)^{-3/4-3\delta} + \exp\{-C_1\tau (PQ)^{1/4+\delta}\} + (PQ)^{-3}\\
\leqslant & \exp\{-R_1^2/2\}+ 2(PQ)^{-3/4-3\delta} + \exp\{-C_1\tau (PQ)^{1/4+\delta}\} + (PQ)^{-3}\\
\leqslant & C_4(PQ)^{-3/4-3\delta}.
\end{align*}
Therefore, for any $(p, q)\in\mathcal W_\delta$,
$$P[G(p_0, q_0) \leqslant G(p, q)]\leqslant P[G(p_0, q_0) \leqslant \tau\lambda (1 + \phi)/2]+P[G(p, q) \geqslant \tau\lambda (1 + \phi)/2]\leqslant 2C_4(PQ)^{-3/4-3\delta}.$$
Using the union bound we have
$$P[G(p_0, q_0)\leqslant \max_{(p, q)\in\mathcal W_\delta} G(p, q)] \leqslant 2|\mathcal W_\delta|C_4(PQ)^{-3/4}-3\delta,$$
where $|\mathcal W_\delta|$ is the number of wrong configurations. Applying $|\mathcal W_\delta|< |d(P)||d(Q)|$ yields the theorem.

\section{Proof of Lemma~\ref{lem:recoverable}}
All entries of $P_\Omega\bm Y$ is equivalent to $\mathcal R_{p_0, q_0}[P_\Omega\bm Y]$ has no missing column or missing row.
Noticing that $\mathcal R_{p_0, q_0}[P_\Omega \bm Y]$ is a $p_0q_0\times p_0^*q_0^*$ matrix with entries observed IID with probability $\tau$, we have
$$P[\mathcal R_{p_0, q_0}[P_\Omega \bm Y]\text{ has no missing column}]\geqslant [1-(1-\tau)^{p_0q_0}]^{p_0^*q_0^*}\geqslant 1 - PQ\dfrac{(1-\tau)^{p_0q_0}}{p_0q_0}.$$
The right hand side is a increasing function of $p_0q_0$ and $\tau$. From Assumption~\ref{assump:true-config} we know $p_0q_0\geqslant (PQ)^{1/4+\delta}$ and Assumption~\ref{assump:observing-rate} gives $\tau > \log PQ\cdot (PQ)^{-2\delta}$. Therefore,
$$\dfrac{(1-\tau)^{p_0q_0}}{p_0q_0}=\exp\{p_0q_0\log (1-\tau) - \log p_0q_0\}\leqslant \exp\left\{-(PQ)^{1/4-\delta}\log PQ - \left(\dfrac{1}{4}+\delta\right)\log PQ\right\}.$$
Using a similar argument for missing row, the Lemma follows immediately. 

\section{Proof of Theorem~\ref{thm:recovery-error}}
We first restate the major result of \cite{gunasekar2013noisy} for rank-1 case in the following lemma.
\begin{lem}\label{lem:als-error}
Let $\bm M=\lambda uv^T$ be a rank-1 $m\times n$ ($m\leqslant n$), incoherent matrix with both $u$ and $v$ being $\mu$ incoherent. Furthermore, it is assumed that the noise matrix $\bm N$ satisfies $\|P_\Omega \bm N\|_S/\tau \leqslant C \lambda$ for some constant $C$. Additionally, let each entry of $\tilde {\bm M} = \bm M + \bm N$ be observed IID with probability 
$$\tau > C\dfrac{\mu^4\log n \log \frac{\lambda}{\epsilon}}{m},$$
where $C>0$ is a global constant. Then with high probability, the output matrix $\hat{\bm M}$ from Algorithm~\ref{alg:alternating-minimization} satisfies
$$\dfrac{1}{\sqrt{mn}}\|\bm M - \hat{\bm M}\|_F\leqslant \epsilon + 20\mu\dfrac{\|P_\Omega \bm N\|_S}{|\Omega|}.$$

\end{lem}

Apply Lemma~\ref{lem:als-error} directly and we have with high probability
$$\|\bm X - \hat{\bm X}\|_F\leqslant \lambda\sqrt{PQ}\exp\left\{-C\dfrac{\tau (p_0q_0\wedge p_0^*q_0^*)}{\mu^4\log PQ}\right\}+20\mu \sigma \dfrac{\|\mathcal R_{p_0, q_0}[P_\Omega\bm E]\|_S}{|\Omega|}.$$
From Lemma~\ref{lem:restate-bound}, we have with high probability, 
$$\|\mathcal R_{p_0, q_0}[P_\Omega\bm E]\|_S \leqslant C_1\sqrt{\tau}\sqrt{\log PQ}\cdot (PQ)^{3/4}(p_0q_0\wedge p_0^*q_0^*)^{-1},$$
for some constant $C_1$, and with high probability
$$|\Omega| \geqslant  C_2 \tau PQ$$
for some constant $0<C_2<1$. 
Therefore, with high probability, 
$$\|\bm X - \hat{\bm X}\|_F\leqslant \lambda\sqrt{PQ}\exp\left\{-C\dfrac{\tau (p_0q_0\wedge p_0^*q_0^*)}{\mu^4\log PQ}\right\}+\dfrac{20C_1}{C_2}\mu \sigma\tau^{-1/2} \sqrt{\log PQ}\cdot (PQ)^{-1/4}(p_0q_0\wedge p_0^*q_0^*)^{-1}.$$
\end{document}